\documentclass{article}

\PassOptionsToPackage{numbers, compress}{natbib}
\usepackage[final]{neurips_2018}
\usepackage[utf8]{inputenc} % allow utf-8 input
\usepackage[T1]{fontenc}    % use 8-bit T1 fonts
\usepackage{hyperref}       % hyperlinks
\usepackage{url}            % simple URL typesetting
\usepackage{booktabs}       % professional-quality tables
\usepackage{amsfonts}       % blackboard math symbols
\usepackage{nicefrac}       % compact symbols for 1/2, etc.
\usepackage{microtype}      % microtypography

\usepackage{graphicx}
\usepackage{subcaption}
\usepackage{dblfloatfix}    % To enable figures at the bottom of page

\usepackage{color}
\usepackage{tikz}
\usetikzlibrary{automata,arrows,positioning,calc}
\usepackage{tkz-graph}

\usepackage{algorithm}% http://ctan.org/pkg/algorithms
\usepackage[noend]{algpseudocode}

\usepackage{amsmath, amsthm, amssymb}
\usepackage{url}
\usepackage{sidecap}

\input{Definitions.tex}

\newcommand{\xdim}{d}

\newcommand{\nsamples}{T}

\newcommand{\rmax}{r_{\text{max}}}

\newcommand{\window}{w}

\newcommand{\stepsize}{\alpha}

\newcommand{\counts}{{\mathbf{c}}}
\newcommand{\deltaprob}{p}
\newcommand{\cmax}{c_{\text{max}}}

\newcommand{\avec}{\mathbf{a}}
\newcommand{\bvec}{\mathbf{b}}
\newcommand{\dvec}{\mathbf{d}}

\newcommand{\svec}{\mathbf{s}}

\newcommand{\vvec}{\mathbf{v}}
\newcommand{\wvec}{\mathbf{w}}
\newcommand{\xvec}{\mathbf{x}}
\newcommand{\zvec}{\mathbf{z}}

\newcommand{\weights}{\vec{w}}

\newcommand{\epsilonvec}{{\boldsymbol{\epsilon}}}

\newcommand{\nubarvec}{\bar{{\boldsymbol{\nu}}}}
\newcommand{\zbarvec}{\bar{{\mathbf{z}}}}

\newcommand{\Amat}{\mathbf{A}}

\newcommand{\Bmat}{\mathbf{B}}
\newcommand{\Cmat}{\mathbf{C}}

\newcommand{\Gmat}{\mathbf{G}}

\newcommand{\Actions}{\mathcal{A}}

\newcommand{\Pfcn}{\mathrm{Pr}}

\newcommand{\Rfcn}{r}

\newcommand{\States}{\mathcal{S}}
\newcommand{\dpi}{\mathbf{d}_\pi}

\newcommand{\E}{\mathbb{E}}
\newcommand{\Var}{\mathbb{V}}

\newcommand{\defeq}{\mathrel{\overset{\makebox[0pt]{\mbox{\normalfont\tiny\sffamily def}}}{=}}}
\newcommand{\wspace}{\mathcal{W}}

\title{Context-Dependent Upper-Confidence Bounds for Directed Exploration}

\author{
  Raksha Kumaraswamy$^1$, Matthew Schlegel$^1$, Adam White$^{1,2}$, Martha White$^1$\\
  $^1$Department of Computing Science, University of Alberta; $^2$DeepMind\\
  \texttt{\{kumarasw, mkschleg\}@ualberta.ca, adamwhite@google.com, whitem@ualberta.ca}
}

\begin{document}

\maketitle

\begin{abstract}
  Directed exploration strategies for reinforcement learning are critical for learning an optimal policy in a minimal number of interactions with the environment. Many algorithms use optimism to direct exploration, either through visitation estimates or upper-confidence bounds, as opposed to data-inefficient strategies like $\epsilon$-greedy that use random, undirected exploration. Most data-efficient exploration methods require significant computation, typically relying on a learned model to guide exploration. Least-squares methods have the potential to provide some of the data-efficiency benefits of model-based approaches---because they summarize past interactions---with the computation closer to that of model-free approaches.  In this work, we provide a novel, computationally efficient, incremental exploration strategy, leveraging this property of least-squares temporal difference learning (LSTD). We derive upper-confidence bounds on the action-values learned by LSTD, with context-dependent (or state-dependent) noise variance. Such context-dependent noise focuses exploration on a subset of variable states, and allows for reduced exploration in other states.
  We empirically demonstrate that our algorithm can converge more quickly than other incremental exploration strategies using confidence estimates on action-values.
\end{abstract}

\vspace{-0.3cm}
\section{Introduction}

Exploration is crucial in reinforcement learning, as the data gathering process significantly impacts the optimality of the learned policies
and values. The agent needs to balance the amount of time taking exploratory actions
to learn about the world, versus taking actions to maximize cumulative rewards. If the agent explores insufficiently,
it could converge to a suboptimal policy; exploring too conservatively, however, results in many suboptimal decisions.
The goal of the agent is \emph{data-efficient exploration}: to minimize how many samples are wasted in exploration,
particularly exploring parts of the world that are known, while still ensuring convergence to the optimal policy.

To achieve such a goal, directed exploration strategies are key.
Undirected strategies, where random actions are taken such as in $\epsilon$-greedy,
are a common default.
%In small domains these methods are guaranteed to find an optimal policy \citep{singh2000convergence}, because the agent is guaranteed to visit the entire space---albeit in many many steps.
In small domains these methods are guaranteed to find an optimal policy \citep{singh2000convergence}, because the agent is guaranteed to visit the entire space---but may take many many steps to do so, as undirected exploration can interfere with improving policies in incremental control.
In this paper we explore the idea of constructing confidence intervals around the agent's value estimates. The agent can use these learned confidence intervals to select actions with the highest upper-confidence bound ensuring actions selected are of high value or whose values are highly uncertain. This optimistic approach is promising for directed exploration, but as yet there are few such methods that are model-free, incremental and computationally efficient.
% MARTHA: Add a sentence explaining why this is an optimistic approach
%---no more than quadratic in the features, which we motivate is the minimal required computation. We describe below what has been done so far towards this question, and what we contribute
%Regardless, such an approach affords an adaptive exploration mechanism driven by the variance in the learning process, while still ensuring convergence to a deterministic greedy action selection rule. These ideas have been most extensively investigated in the bandit literature, however, there have not been many notable attempts to employ statistical, supervised approaches to exploration in reinforcement learning.

%systematic exploration
%with actions taken intentionally to either gain information maximally or maximize rewards.

%Martha: what about bonus. They are the oppositive of counts and novelty. The longer its been since you have seen a state the more the agent wants to revisit it. its statistical, but different. RESOLVED: not doing this for now
Directed exploration strategies have largely been explored under the framework of ``optimism in the face of uncertainty'' \cite{kaelbling1996reinforcement}.
These can generally be categorized into count-based approaches and confidence-based approaches. Count-based approaches estimate the ``known-ness'' of a state, typically by maintaining counts for finite state-spaces \citep{kearns2002near,brafman2003rmax,strehl2004exploration,strehl2006pac,szita2010model}
and extensions on counting for continuous states \citep{kakade2003exploration,jong2007model,nouri2009multi,li2009online,pazis2013pac,kawaguchiAAAI2016,ostrovski2017count,martin2017count}.
Confidence interval estimates, on the other hand, depend on variance of the target, not just on visitation frequency for states. Confidence-based approaches can be more data-efficient for exploration, because the agent can better direct exploration where the estimates are less accurate. The majority of confidence-based approaches compute confidence intervals on model parameters, both for finite state-spaces \citep{kaelbling1993learning,wiering1998efficient,kearns2002near,brafman2003rmax,auer2006logarithmic,bartlett2009regal,jaksch2010near,szita2010model,osband2013efficient} and continuous state-spaces \citep{jung2010gaussian,ortner2012online,grande2014sample,abbasiyadkori2014bayesian,osband2017why}.
%There has been some work maintaining confidence intervals on value functions for finite state-spaces \citep{meuleau1999exploration}.
There is early work quantifying uncertainty for value estimates directly for finite state-spaces \citep{meuleau1999exploration},
describing the difficulties with extending the local measures of uncertainty from the bandit literature to RL, since there are long-term dependencies.

% TODO: this paper ruckstiess2008state is about putting noise on the action-selection, but rather than random-noise, noise that is dependnet on the state in some way. This is another, directed exploration strategy, but doesn't really fit into what weve described here.
These difficulties suggest why using confidence intervals directly on value estimates for exploration in RL has been less explored, until recently.
More approaches are now being developed that maintain confidence intervals on the value function for continuous state-spaces, by maintaining a distribution over value functions \citep{grande2014sample,osband2016generalization}, or by maintaining a randomized set of value functions from which to sample \citep{white2010interval,osband2016generalization,osband2016deep,plappert2017parameter,moerland2017efficient}.
%, and ignore (low-variance) regions for which the agent has a confident estimate.
Though significant steps forward, these approaches have limitations particularly in terms of computational efficiency.
%The approaches computing confidence intervals, on the other hand, can require significant memory and computation.
%, to store and access many value functions or model parameters.
Delayed Gaussian Process Q-learning (DGPQ) \citep{grande2014sample} requires updating two Gaussian processes, which is cubic in the number of basis vectors for the Gaussian process.
%Further, this approaches requires that value parameterization to be a Gaussian process, which can be difficult to tune.
RLSVI \citep{osband2016generalization} is relatively efficient, maintaining a Gaussian distribution over parameters with Thompson sampling to get randomized values. Their staged approach for finite-horizon problems, however, does not allow for value estimates to be updated online, as the value function is fixed per episode to gather an entire trajectory of data.
\citet{moerland2017efficient}, on the other hand, sample a new parameter vector from the posterior distribution each time an action is considered, which is expensive. The bootstrapping approaches can be efficient, as they simply have to store several value functions, either for training on a bootstrapped subset of samples---such as in Bootstrapped DQN \citep{osband2016deep}---or for maintaining a moving bootstrap around the changing parameters themselves, for UCBootstrap \citep{white2010interval}. For both of these approaches, however, it is unclear how many value functions would be required, which could be large depending on the problem.

In this paper, we provide an incremental, model-free exploration algorithm with fast converging upper-confidence bounds, called UCLS: Upper-Confidence Least-Squares. We derive the upper-confidence bounds for Least-Squares Temporal Difference learning (LSTD), taking advantage of the fact that LSTD has an efficient summary of past interaction to facilitate computation of confidence intervals. Importantly, these upper-confidence bounds have context-dependent variance, where variance is dependent on state rather than a global estimate, focusing exploration on states with higher-variance. Computing confidence intervals for action-values in RL has remained an open problem, and we provide the first theoretically sound result for obtaining upper-confidence bounds for policy evaluation under function approximation, without making strong assumptions on the noise.
We demonstrate in several simulated domains that UCLS outperforms
%Sarsa with optimistic initialization,
%RAKSHA: added RLSVI, and modfied second experiment pitch
%DGPQ, UCBootstrap and a simplified version of UCLS that uses a global variance estimate, rather than context-dependent variance.
DGPQ, UCBootstrap, and RLSVI. We also empirically show the benefit of using UCLS to a simplified version that uses a global variance estimate, rather than context-dependent variance.

\vspace{-0.3cm}
\section{Background}

We focus on the problem of learning an optimal policy for a Markov decision process, from on-policy interaction.
A Markov decision process consists of $(\States, \Actions, \Pfcn, \Rfcn,\gamma)$
where
$\States$ is the set of states;
$\Actions$ is the set of actions;
$\Pfcn: \States \times \Actions \times \States \rightarrow [0,\infty)$ provides the transition probabilities;
$\Rfcn: \States \times \Actions \times \States \rightarrow \RR$ is the reward function;
and $\gamma:  \States \times \Actions \times \States \rightarrow [0,1]$ is the transition-based discount function which enables either continuing or episodic problems to be specified \citep{white2017unifying}.
On each step, the agent selects action $A_t$ in state $S_t$,  and transitions to $S_{t+1}$, according to $\Pfcn$, receiving reward
$R_{t+1} \defeq \Rfcn(S_t, A_t, S_{t+1})$ and discount $\gamma_{t+1} \defeq \gamma(S_t, A_t, S_{t+1})$.
%$\Pfcn(S_t,A_t,s')$ is the probability of transitioning from state $s$ into state $s'$ when taking action $a$,
%receiving reward $\Rfcn(s,a,s')$;
For a policy $\pi: \States \times \Actions \rightarrow [0,1]$, where $\sum_{a\in \Actions} \pi(s,a) = 1 \; \forall s \in \States$,
the value at a given state $s$, taking action $a$,
is the expected discounted sum of future rewards, with actions selected according to $\pi$ into the future,
\small
\begin{align*}
Q^\pi(s, a) = \E\Big[R_{t+1} + \gamma_{t+1} \sum_{a \in \Actions} \pi(S_{t+1}, a) Q^\pi(S_{t+1}, a) \Big| S_t = s, A_t = a \Big]
\end{align*}
\normalsize
%%
%%
%\begin{align*}
%&Q^\pi(s_t, a_t) = \E\Big[R_{t+1} + \\
%&
%\ \ \ \ \ \ \ \ \ \ \ \ \  \gamma_{t+1} \sum_{a \in \Actions} \!\!\pi(S_{t+1}, a) Q^\pi(S_{t+1}, a) \Big| S_t \!=\! s_t, A_t \!=\! a_t \Big]
%.
%\end{align*}
%%
%%
For problems in which $Q^\pi$ can be stored in a table, a fixed point for the action-values $Q^\pi$ exists for a given $\pi$. In most domains, $Q^\pi$ must be approximated by $Q_{\weights}^\pi$, parametrized by $\weights \in \wspace \subset \RR^{\xdim}$.
%the goal is typically to find the projected fixed point, with projection onto the space of action-value functions representable by the parameterization \citep{sutton2009fast}.

In the case of linear function approximation, state-action features $\xvec(s_t, a_t)$ are used to
approximate action-values $Q_{\weights}^\pi(s_t,a_t) = \xvec(s_t, a_t)^\top \weights$.
The weights $\wvec$ can be learned with a stochastic approximation algorithm, called temporal difference (TD) learning \citep{sutton1988learning}.
The TD update \citep{sutton1988learning} processes samples one at a time,
%\begin{equation*}
$\wvec = \wvec + \stepsize \delta_t \zvec_t$,
%\end{equation*}
%
%
with $\delta_t \defeq R_{t+1} + \gamma_{t+1} \xvec_{t+1}^\top \wvec - \xvec_{t}^\top \wvec$ for $\xvec_{t} \defeq \xvec(S_t, A_t)$. The
eligibility trace $\zvec_t = \xvec_t + \gamma_{t+1} \lambda \zvec_{t-1}$ facilitates multi-step updates via an exponentially weighted memory of previous feature activations decayed by $\lambda \in [0,1]$ and $\zvec_0 = \zerovec$. Alternatively, we can directly compute the weight vector found by TD using least-squares temporal difference learning (LSTD) \citep{bradtke1996linear}.
%Consider a trajectory sampled according to $\pi$: $S_0, A_0, R_1, S_1, A_1, \ldots, S_{\nsamples-1}, A_{\nsamples-1}, R_\nsamples, S_\nsamples, A_\nsamples$.
%The eligibility traces places a bias-variance role, as well as a credit-assignment role, and makes learning more efficient \citep{sutton1988learning}.
The LSTD solution is more data-efficient, and can avoid the need to tune TD's stepsize parameter $\stepsize > 0$. The LSTD update can be efficiently computed incrementally without approximation or storing the data \citep{bradtke1996linear,boyan2002technical}, by maintaining a
matrix $\Amat_\nsamples$ and vector $\bvec_\nsamples$,
\begin{equation}
\Amat_\nsamples \defeq \frac{1}{\nsamples} \sum_{t=0}^{\nsamples-1} \zvec_t (\xvec_t - \gamma_{t+1} \xvec_{t+1})^\top  \hspace{1.0cm}
\bvec_\nsamples \defeq \frac{1}{\nsamples} \sum_{t=0}^{\nsamples-1} \zvec_t R_{t+1} \label{b-vec-update}
\end{equation}
The value function approximation at time step $\nsamples$ is the weights that satisfy the linear system $\Amat_\nsamples \wvec = \bvec_\nsamples$.
In practice, the inverse of the matrix $\Amat^\inv$ is maintained using a Sherman-Morrison update, with a small regularizer $\eta$ added to the matrix $\Amat$ to guarantee invertibility \citep{szepesvari2010algorithms}.

%In policy iteration, the agent incrementally estimates the action-values for its current policy, and then acts greedily with respect to those action-values.
%The agent needs to explore sufficiently to ensure it does not converge to a suboptimal policy.
%Greedy action selection can prevent sufficient exploration, because values could be artificially low due to variability. For example, stochastically, a low reward could be observed the first time an action is taken in a state.
%%or because bootstrapped value estimates from a state could also be low, further compounding the issue.
%For model-free approaches, therefore, it is common to make the action-values optimistic, either with optimistic initialization
%or with upper-confidence values. In this way, the agent is inclined to systematically (and exhaustively) visit the space.

%MARTHA: Add a short paragraph (in a teachy background section like way) Explaining how CIs could be added to TD and LSTD and used to select actions. This would assume some oracle computed the correct CI, but it would be instructive. Then say how it encourages systematic exploration that is sensitive to the variance due to the MDP and the learning algorithm.

One approach to ensure systematic exploration is to initialize the agent's value estimates optimistically.
%Optimistic initialization---though it can be an effective heuristic---is difficult to use for certain settings.
The action-value function must be initialized to predict the maximum possible return (or greater) from each state and action. For example, for cost-to-goal problems, with -1 per step, the values can be initialized to zero. For continuing problems, with constant discount $\gamma_c < 1$, the values can be initialized to $\textrm{G}_\textrm{max} = \textrm{R}_{\text{max}} / (1- \gamma_c)$, if the maximum reward $\textrm{R}_{\text{max}} $ is known. For fixed features that are non-negative and encode locality---such as tile coding or radial basis functions---the weights $\wvec$ can be simply set to $\textrm{G}_\textrm{max}$, to make $Q_{\weights}$ optimistic.

More generally, however, it can be problematic to use optimistic initialization.
Optimistic initialization assumes the beginning of time is special---a period when systematic exploration should be performed after which the agent should more or less exploit its current knowledge. Many problems are non-stationary---or at least benefit from a tracking approach due to aliasing caused by function approximation---and benefit from continual exploration. Further, unlike for fixed features, it is unclear how to set and maintain initial values at $\textrm{G}_\textrm{max}$ for learned features, such as with neural networks.  Optimistic initialization is also not straightforward for algorithms like LSTD, which completely overwrite the estimate $\weights$ on each step with a closed-form solution.
In fact, we have found that this issue with LSTD has been obfuscated, because the regularizer $\eta$ has inadvertently played a role in providing optimism (see Appendix \ref{app_issues}). Rather, to use optimism in LSTD for control, we need to explicitly compute upper-confidence bounds.

Confidence intervals around action-values, then, provide another mechanism for exploration in reinforcement learning. Consider action selection with explicit confidence intervals around mean estimates $\hat{Q}_{\weights}(S_t,A_t)$, with estimated radius $\hat{U}(S_t,A_t)$. The action selection is greedy w.r.t. to these optimistic values, $\argmax_a \hat{Q}_{\weights}(S_t,a) + \hat{U}(S_t,a)$, which provides a high-confidence upper bound on the best possible value for that action.
 %Such a strategy was explicitly pursued in UCBootstrap \citep{white2010interval}, and more implicitly in methods that keep an interval around parameter estimates \citep{grande2014sample,osband2016generalization}.
 The use of upper-confidence bounds on value estimates for exploration has been well-studied and motivated theoretically in online learning \citep{chu2011contextual}. In reinforcement learning, there have only been a few specialized proofs for particular algorithms using optimistic estimates \citep{grande2014sample,osband2016generalization}, but the result can be expressed more generally by using the idea of stochastic optimism. We extract the central argument by \citet{osband2016generalization} to provide a general Optimistic Values Theorem in Appendix \ref{app_optimisticvalues}. In particular, similar to online learning, we can guarantee that greedy-action selection according to upper-confidence values will converge to the optimal policy, if the confidence interval radius shrinks to zero, if the algorithm to estimate action-values for a policy converges to the corresponding actions and if upper-confidence estimates are stochastically optimal---remain above the optimal action-values in expectation.

Motivated by this result, we pursue principled ways to compute upper-confidence bounds for the general, online reinforcement learning setting.
We make a step towards computing such values incrementally, under function approximation, by providing upper-confidence bounds for value estimates made by LSTD, for a fixed policy. We approximate these bounds to create a new algorithm for control---called Upper-Confidence-Least-Squares (UCLS).

\vspace{-0.3cm}
\section{Estimating Upper-Confidence Bounds for Policy Evaluation using LSTD}\label{sec_UCBound}

Consider the goal of obtaining a confidence interval around value estimates learned incrementally by LSTD for a fixed policy $\pi$.
The value estimate is $\xvec^\top \wvec$ for state-action features $\xvec$ for the current state and action.
We would like to guarantee, with probability $1-\deltaprob$ for a small $\deltaprob > 0$, that the confidence interval around this estimate contains the value $\xvec^\top \wvec^*$ given by the optimal $\wvec^* \in \wspace$. To estimate such an interval without parametric assumptions, we use Chebyshev's inequality which---unlike other concentration inequalities like Hoeffding or Bernstein---does not require independent samples.

To use this inequality, we need to determine the variance of the estimate $\xvec^\top \wvec$;
% and so need to consider the source of this variability.
the variance of the estimate, given $\xvec$, is due to the variance of the weights.
Let $\weights^*$ be fixed point solution for the projected Bellman operator for the $\lambda$-return---the TD fixed point, for a fixed policy $\pi$.
To characterize the noise for this optimal estimator, let $\nu_t$ be the TD-error for the optimal weights $\weights^*$, where
\begin{equation}
r_{t+1} = (\xvec_t - \gamma \xvec_{t+1})^\top \weights^* + \nu_t \label{eq_nu}
\hspace{1.0cm} \text{ with } \E[\nu_t \zvec_t] = 0.
\end{equation}
The expectation is taken across all states weighted by the sampling distribution, typically the stationary distribution $\dpi: \States \rightarrow [0,\infty)$ or in the off-policy case the stationary distribution of the behaviour policy. We know that $\E[\nu_t \zvec_t] = 0$, by the definition of the
%TD-fixed point.
Projected Bellman Error fixed point.

This noise $\nu_t$ is incurred from the variability in the reward, the variability in the transition dynamics
and potentially the capabilities of the function approximator.
A common assumption---when using linear regression for contextual bandits \citep{li2010acontextual} and for reinforcement learning \citep{osband2016generalization}---is that
the variance of the target is a constant value $\sigma^2$ for all contexts $\xvec$. Such an assumption, however, is likely to produce
larger confidence intervals than necessary. For example, consider a one-state world with two actions,
where one action has a high variance reward and the other has a lower variance reward (see Appendix \ref{app_issues}, Figure \ref{fig_onestate}).
A global sample variance will encourage both actions to be taken many times. For data-efficient exploration, however, the agent should take the high-variance action more, and only needs a few samples from the low-variance action.

We derive a confidence interval for LSTD, in Theorem \ref{thm_main}.
%We derive a confidence interval for LSTD without making the simplifying assumption of a global variance, and instead allow for the
%variance of $\nu_t$ to be context-dependent. We provide this result in Theorem \ref{thm_main}.
We also derive the confidence interval assuming a global variance in Corollary \ref{cor_global}, to provide a comparison. We compare to using this global-variance upper-confidence bound in our experiments, and show that it results in significantly worse performance than using a context-dependent variance.
Note that we do not assume $\Amat_\nsamples$ is invertible; if we did, the big-O term in \eqref{UCLS-upperbound} below would disappear. We include this term for preciseness of the result---even though we will not estimate it---because for smaller $\nsamples$, $\Amat_\nsamples$ is unlikely to be invertible. However, we expect this big-O term to get small quickly, and be dominated by the other terms. In our algorithm, therefore, we ignore the big-O term.
%\footnote{The only model-free algorithm that achieves a regret bound is RLSVI, but that bound is restricted to the finite horizon, batch, tabular setting. It would be a substantial breakthrough to provide such a regret bound, and is beyond the scope of this work.}

%% TODO: this is not relevant, so I'm ignoring it
%this bound can be extended to comparison to the true expected return using
%%
%\begin{equation*}
%\left| \xvec_t^\top \weights_t - \E[G_t | s_t] \right| \le \left| \xvec_t^\top \weights^* - \E[G_t | s_t] \right| + \left| \xvec_t^\top \weights_t - \xvec_t^\top \weights^* \right|
%\end{equation*}
%%
%The first error term on the right-hand side reflects the inherent error in our function class, for estimating
%the expected return.

% \input{theorem-org}
%!TEX root = paper.tex

\newcommand{\AmatPsuedo}{\Amat_\nsamples}
\newcommand{\AmatPsuedoInv}{\Amat_\nsamples^{+}}
\newcommand{\AmatPsuedoInvTrans}{\Amat_\nsamples^{+\top}}
\newcommand{\remainvec}{\epsilonvec^*_\nsamples}
\newcommand{\remainvectrans}{\epsilonvec^{*\top}_\nsamples}
\newcommand{\remainmat}{\boldsymbol{\Sigma}^*_\nsamples}
\newcommand{\mubarvecremain}{\boldsymbol{\mu}^*_\nsamples}
\newcommand{\mubarvecremaintrans}{\boldsymbol{\mu}^{*\top}_\nsamples}

% MARTHAC: I don't think we need to make this assumption yet. Rather, we'll want to characterize the error term as T gets larger. I'm commenting this for now, until we see what we need
%\begin{assumption}[Existence of $\Amat_\nsamples^\inv$]
%With high probability, if $\nsamples$ is sufficiently large, $\Amat_\nsamples$ is invertible. In order to ensure invertibility we use $\AmatPsuedo = \Amat_\nsamples + \eta \eye$, where $\Amat_\nsamples^\inv$ is needed, as, $\lim_{\nsamples \rightarrow \infty}\AmatPsuedoInv \Amat_\nsamples \preceq \kappa(\Amat_\nsamples)$. (proof in appendix)
%\end{assumption}

%
\begin{theorem}\label{thm_main}
Let
$\nubarvec_\nsamples \defeq \tfrac{1}{\nsamples} \sum_{t=0}^{\nsamples-1} \zvec_t \nu_t$
and $\weights_\nsamples = \AmatPsuedoInv \bvec_\nsamples$ where
%$\Vmat_\nsamples \defeq  \Var[\AmatPsuedoInv \nubarvec_\nsamples]$ and
%$\mubarvec_\nsamples \defeq \E[\AmatPsuedoInv \nubarvec_\nsamples]$ where
$\AmatPsuedoInv$ is the pseudoinverse of $\AmatPsuedo$. Let $\remainvec \defeq (\AmatPsuedoInv\Amat_\nsamples -\eye )\weights^*$ reflect the degree to which $\AmatPsuedo$ is not invertible; it is zero when $\AmatPsuedo$ is invertible.
Assume that the following are all finite: $\E[\AmatPsuedoInv\nubarvec_\nsamples + \remainvec] $, $\Var[\AmatPsuedoInv\nubarvec_\nsamples + \remainvec]$ and all state-action features $\xvec$.
With probability at least $1- \deltaprob$, given state-action features $\xvec$,
\begin{align}
\xvec^\top \weights^*
%&\le  \xvec^\top \weights_\nsamples + \tfrac{1}{\sqrt{\deltaprob}} \sqrt{\xvec^\top \Vmat_\nsamples \xvec } + \sqrt{\xvec^\top \mubarvec_\nsamples \mubarvec_\nsamples^\top \xvec } + O(\E[(\xvec^\top\remainvec)^2]) \nonumber\\
 &\le  \xvec^\top \weights_\nsamples + \sqrt{\tfrac{\deltaprob+1}{\deltaprob}}\sqrt{\xvec^\top \E[\AmatPsuedoInv\nubarvec_\nsamples \nubarvec_\nsamples^\top \AmatPsuedoInvTrans]  \xvec } + O\left(\E[(\xvec^\top\remainvec)^2]\right) \label{UCLS-upperbound}
% &\le  \xvec^\top \weights_\nsamples + \max\left\{\tfrac{1}{\sqrt{\deltaprob}}, 2\right\}\sqrt{\xvec^\top \E[\Amat_\nsamples^\inv\nubarvec_\nsamples \nubarvec_\nsamples^\top \Amat_\nsamples^{-\top}]  \xvec } \label{UCLS-upperbound}
\end{align}
\end{theorem}
\begin{proof}
First we compute the mean and variance for our learned parameters.
Because  $r_{t+1} = (\xvec_t - \gamma \xvec_{t+1})^\top \weights^* + \nu_t$,
%For $\rho(s) = 1/ (\text{min visitation count of a feature for $s$})$,
%
\begin{align*}
\weights_\nsamples &=  \left(\tfrac{1}{\nsamples} \sum_{t=0}^{\nsamples-1} \zvec_t (\xvec_t - \gamma \xvec_{t+1})^\top \right)^{-1}
\left( \tfrac{1}{\nsamples}\sum_{t=0}^{\nsamples-1} \zvec_t r_{t+1} \right)\\
&= \AmatPsuedoInv
\left(\tfrac{1}{\nsamples}\sum_{t=0}^{\nsamples-1} \zvec_t ((\xvec_t - \gamma \xvec_{t+1})^\top \weights^* + \nu_t) \right)\\
&= \AmatPsuedoInv\Amat_\nsamples \weights^* + \AmatPsuedoInv\left(\tfrac{1}{\nsamples}\sum_{t=0}^{\nsamples-1} \zvec_t  \nu_t \right)\\
&= \weights^* + \AmatPsuedoInv \nubarvec_\nsamples + \remainvec
%+ \mathcal{O}((\kappa(\Amat_\nsamples)-1)\weights^*)
\end{align*}
This estimate has a small amount of bias, that vanishes asymptotically.
But, for a finite sample,
\begin{equation*}
%\E\left[\Amat_\nsamples^\inv\left(\tfrac{1}{\nsamples}\sum_{t=0}^{\nsamples-1} \zvec_t  \nu_t \right)\right] \neq \E[\Amat_\nsamples^\inv] \E\left[\tfrac{1}{\nsamples}\sum_{t=0}^{\nsamples-1} \zvec_t  \nu_t \right] = \zerovec
\E\left[\AmatPsuedoInv\left(\tfrac{1}{\nsamples}\sum_{t=0}^{\nsamples-1} \zvec_t  \nu_t \right)\right] \neq \E[\AmatPsuedoInv] \E\left[\tfrac{1}{\nsamples}\sum_{t=0}^{\nsamples-1} \zvec_t  \nu_t \right] = \zerovec
.
\end{equation*}
Further, because $\AmatPsuedo$ may not be invertible, there is an additional error $\remainvec$ term which will vanish with enough samples, i.e., once $\AmatPsuedo$ can be guaranteed to be invertible.
%Further on, we will use $\delta_t$ in place of $\nu_t$.

For covariance, because
\begin{align*}
\weights_\nsamples - \E[\weights_\nsamples]
&= \left( \weights^* + \AmatPsuedoInv \nubarvec_\nsamples + \remainvec \right) - \E\left[\weights^* + \AmatPsuedoInv \nubarvec_\nsamples + \remainvec) \right]\\
&= \AmatPsuedoInv \nubarvec_\nsamples + \remainvec  - \E\left[\AmatPsuedoInv \nubarvec_\nsamples + \remainvec \right]
\end{align*}
the covariance of the weights is
\begin{align*}
\Var[\weights_\nsamples]
&= \Var\left[\AmatPsuedoInv \nubarvec_\nsamples + \remainvec  \right]
%= \E[\AmatPsuedoInv\nubarvec_\nsamples  \nubarvec_\nsamples^\top \AmatPsuedoInvTrans]
%- \E[\AmatPsuedoInv\nubarvec_\nsamples ] \E[\AmatPsuedoInv\nubarvec_\nsamples ]^\top
\end{align*}
%
%For convenience, we will use $\remainmat = \AmatPsuedoInv \nubarvec_\nsamples \remainvectrans + \remainvec (\AmatPsuedoInv \nubarvec_\nsamples)^\top + \remainvec\remainvectrans$.
%where $\Cmat_\nsamples = \tfrac{1}{\nsamples}\nubarvec_\nsamples  \nubarvec_\nsamples^\top$ is the covariance matrix weighted by the state-dependent noise $\nu_t$.
%The middle term is a weighted covariance matrix.
%%
%\begin{equation}
%\Cmat_\nsamples = \frac{1}{\nsamples}\Zmat_\nsamples^\top \diag(\Xmat_t \weightsvar_t) \Zmat_\nsamples
%.
%\end{equation}
%%
%For $\nu_t$ sampled from the same distribution for all states i.e., global noise, this reduces to
%$\Cmat_\nsamples = \frac{\sigma_t^2}{T}\Zmat_\nsamples^\top \Zmat_\nsamples$. The variance parameter $\sigma_t^2$ in the case of global noise, or weights $\weightsvar_t$ in the case of context-dependent noise, can be estimated incrementally as described in the next section.

The goal for computing variances is to use a concentration inequality.
Chebyshev's inequality\footnote{Bernstein's inequality cannot be used here because we do not have independent samples. Rather, we characterize behaviour of the random variable $\weights$, using variance of $\weights$, but cannot use bounds that assume $\weights$ is the sum of independent random variables. The bound with Chebyshev will be loose, but we can better control the looseness of the bound with the selection of $p$ and the constant in front of the square root.} states that for a random variable $X$, if the $\E[X]$ and $\Var[X]$ are bounded, then for any $\epsilon \ge 0$:
\begin{align*}
\Pr \left(\left| X - \E[X] \right| < \epsilon  \sqrt{\Var[X]} \right) &\geq 1 - \frac{1}{\epsilon^2}
%&= \frac{\epsilon^2 - 1}{\epsilon^2}
\end{align*}
%Highlight Chebyshev is the only applicable inequality
% This is said in the footnote
%Although the use of this inequality makes the upper-bound quite loose, it is the applicable inequality here as we wish to bound the deviation of a single random variable from its mean.
%
If we set $\epsilon = \sqrt{1/\deltaprob}$, then this gives
\begin{align*}
\Pr \left(\left| X - \E[X] \right| < \sqrt{\tfrac{1}{\deltaprob}}  \sqrt{\Var[X]} \right) & \geq 1 - \deltaprob
\end{align*}
Now we have characterized the variance of the weights, but what we really want is to characterize the variance of the value estimates. Notice that the variance of the value-estimate, for state-action $\xvec$ is % a given state $s$ with features $\xvec$ is
\begin{align*}
\Var[\xvec^\top \weights_\nsamples | \xvec]
&= \E[\xvec^\top \weights_\nsamples \weights_\nsamples^\top \xvec | \xvec] - \E[\xvec^\top \weights_t | \xvec]^2\\
&= \xvec^\top \left(\E[\weights_\nsamples \weights_\nsamples^\top] - \E[ \weights_\nsamples] \E[ \weights_\nsamples]^\top \right) \xvec \\
&= \xvec^\top \Var[\weights_\nsamples]  \xvec
%&\le 2\xvec^\top \E[\weights_\nsamples\weights_\nsamples^\top]   \xvec
\end{align*}
Therefore, the variance of the estimate is characterized by the variance of the weights.
With high probability,
\begin{align}
\left| \xvec^\top \weights_\nsamples - \xvec^\top \weights^* \right|
&= \left| \xvec^\top (\weights_\nsamples - \E[\weights_\nsamples]) +  \xvec^\top (\E[\weights_\nsamples] - \weights^*) \right| \nonumber\\
&\le \left| \xvec^\top (\weights_\nsamples - \E[\weights_\nsamples]) \right| +  \left|\xvec^\top (\E[\weights_\nsamples] - \weights^*) \right| \nonumber\\
&\le \frac{1}{\sqrt{\deltaprob}} \sqrt{\xvec^\top \Var\left[\AmatPsuedoInv \nubarvec_\nsamples + \remainvec \right]\xvec} + \left|\xvec^\top \E[\AmatPsuedoInv\nubarvec_\nsamples + \remainvec  ] \right| \label{eq_chebyshev_ub}\\
&= \frac{1}{\sqrt{\deltaprob}} \sqrt{\xvec^\top \left(\E\left[\AmatPsuedoInv\nubarvec_\nsamples \nubarvec_\nsamples^\top \AmatPsuedoInvTrans +\remainmat \right] - \mubarvecremain \mubarvecremaintrans \right)\xvec} + \sqrt{\xvec^\top \mubarvecremain \mubarvecremaintrans\xvec} \label{eq_final}
\end{align}
where Equation \ref{eq_chebyshev_ub} uses Chebyshev's inequality, and the last step is a rewriting of Equation \ref{eq_chebyshev_ub} using the definitions $\mubarvecremain \defeq \E[\AmatPsuedoInv\nubarvec_\nsamples + \remainvec]$ and $\remainmat \defeq \AmatPsuedoInv \nubarvec_\nsamples \remainvectrans + \remainvec (\AmatPsuedoInv \nubarvec_\nsamples)^\top + \remainvec\remainvectrans$.

%\iffalse
%Now for $\mubarvecremain = \E[\AmatPsuedoInv\nubarvec_\nsamples + \remainvec  ]$, notice that
%%
%\begin{align*}
%\left|\xvec^\top \mubarvecremain \right| &= \sqrt{\xvec^\top \mubarvecremain \mubarvecremaintrans\xvec}\\
%\xvec^\top \Var\left[\AmatPsuedoInv\nubarvec_\nsamples + \remainvec  \right]\xvec &= \xvec^\top \left(\E\left[\AmatPsuedoInv\nubarvec_\nsamples \nubarvec_\nsamples^\top \AmatPsuedoInvTrans \right] + \E\left[\remainmat\right] - \mubarvecremain \mubarvecremaintrans \right)\xvec
%\end{align*}
%%
%%
%\fi

To simplify \eqref{eq_final}, we need to determine an upper bound for the general formula $c \sqrt{a^2 - b^2} + b$ where $a \ge b \ge 0$. Because $\deltaprob < 1$, we know that $c = \sqrt{1/\deltaprob} \ge 1$. Therefore, the extremal points for $b$, $b = a$ and $b = 0$, both result in an upper bound of $c a$. Taking the derivative of the objective, gives a single stationary point in-between $[0, a]$, with $b = \frac{a}{\sqrt{c^2+1}}$.  The value at this point evaluates to be $a\sqrt{c^2+1}$. Therefore, this objective is upper-bounded by $a\sqrt{c^2+1}$.

Now for $a^2 =  \xvec^\top \E\left[\AmatPsuedoInv\nubarvec_\nsamples \nubarvec_\nsamples^\top \AmatPsuedoInvTrans + \remainmat \right]\xvec$, the term involving $\xvec^\top \E\left[\remainmat \right]\xvec$ should quickly disappear, since it is only due to the potential lack of invertibility of $\AmatPsuedo$. This term is equal to $ \E\left[ 2(\xvec^\top \AmatPsuedoInv \nubarvec_\nsamples) (\xvec^\top\remainvec) + (\xvec^\top\remainvec)^2\right]$, which results in the additional $O(\E[(\xvec^\top\remainvec)^2])$ in the bound.
%%:
%\iffalse
%\begin{align*}
%  \left| \xvec^\top \weights_\nsamples - \xvec^\top \weights^* \right| &\leq \sqrt{1+\frac{1}{p}}\sqrt{\xvec^\top \E\left[\AmatPsuedoInv\nubarvec_\nsamples \nubarvec_\nsamples^\top \AmatPsuedoInvTrans + \remainmat \right]\xvec}
%\end{align*}
%\fi
%%Further, as
% and further as,
% $\lim_{\nsamples\rightarrow\infty} \remainvec \rightarrow 0 \implies \xvec^\top \remainmat \xvec \rightarrow 0$, and is therefore ignored, resulting in the mentioned upper-bound.

%To simplify \eqref{eq_chebyshev_ub}, we need to determine an upper bound for the general formula $c \sqrt{a^2 - b^2} + b$ where $a \ge b \ge 0$. Because $\deltaprob < 1$, we know that $c = \sqrt{1/\deltaprob} \ge 1$. Therefore, the extremal points for $b$, $b = a$ and $b = 0$, both result in an upper bound of $c a$. Taking the derivative of the objective, gives a single stationary point in-between $[0, a]$, with $b = \frac{a}{\sqrt{c^2+1}}$.  The value at this point evaluates to be $a\sqrt{c^2+1}$. Therefore, this objective is upper-bounded by $a\sqrt{c^2+1}$.

%Now we can characterize the upper bound in terms of $\Vmat_\nsamples $ and $\remainmat$. TODO, finish.
\end{proof}

\begin{corollary}\label{cor_global}
Assume that $\nu_t$ are i.i.d., with mean zero and bounded variance $\sigma^2$.
Let $\zbarvec_\nsamples = \tfrac{1}{\nsamples} \sum_{t=0}^{\nsamples-1} \zvec_t $
%%
%%
%\begin{align*}
%\zbarvec_\nsamples &= \tfrac{1}{\nsamples} \sum_{t=0}^{\nsamples-1} \zvec_t
%\end{align*}
%%
and assume that the following are finite: $\E[\remainvec]$, $\Var[\remainvec]$, $\E[\AmatPsuedoInv\zbarvec_\nsamples \zbarvec_\nsamples^\top \AmatPsuedoInvTrans]$ and all state-action features $\xvec$.
With probability at least $1- \deltaprob$, given state-action features $\xvec$,
\begin{align}
&\xvec^\top \weights^* \le  \xvec^\top \weights_\nsamples + \sigma \sqrt{\tfrac{p+1}{p}} \sqrt{\xvec^\top \E[\AmatPsuedoInv\zbarvec_\nsamples \zbarvec_\nsamples^\top \AmatPsuedoInvTrans]\xvec } + O\left(\E[(\xvec^\top\remainvec)^2]\right) \label{global-upperbound}
\end{align}
\end{corollary}
\begin{proof}
The result follows similarly to above, with some simplifications due to global-variance:
\begin{align*}
\E\left[ \AmatPsuedoInv\nubarvec_\nsamples \right]
&= \E\left[\E\left[ \AmatPsuedoInv\nubarvec_\nsamples \Big| S_0, ...., S_\nsamples \right]\right]
= \E\left[\AmatPsuedoInv\tfrac{1}{\nsamples} \sum_{t=0}^{\nsamples-1} \zvec_t \E\left[\nu_t \Big| S_0, ...., S_\nsamples \right]\right] = \zerovec\\
\E[\AmatPsuedoInv\nubarvec_\nsamples  \nubarvec_\nsamples^\top \AmatPsuedoInvTrans]
&= \sigma^2\E[\AmatPsuedoInv\zbarvec_\nsamples  \zbarvec_\nsamples^\top \AmatPsuedoInvTrans]
\end{align*}
\par\vspace{-0.6cm}
%Similarly to above, $\weights_\nsamples = \weights^* + \AmatPsuedoInv\nubarvec_\nsamples$.
%%%
%%%
%%\begin{align*}
%%\weights_\nsamples &= \weights^* + \Amat_\nsamples^\inv\nubarvec_\nsamples
%%\end{align*}
%%%
%With global-variance,
%%
%\begin{equation*}
%\E\left[ \AmatPsuedoInv\nubarvec_\nsamples \right]
%= \E\left[\E\left[ \AmatPsuedoInv\nubarvec_\nsamples \Big| S_0, ...., S_\nsamples \right]\right]
%= \E\left[\AmatPsuedoInv\tfrac{1}{\nsamples} \sum_{t=0}^{\nsamples-1} \zvec_t \E\left[\nu_t \Big| S_0, ...., S_\nsamples \right]\right] = \zerovec
%%&= \E\left[\Amat_\nsamples^\inv\tfrac{1}{\nsamples} \sum_{t=0}^{\nsamples-1} \zvec_t 0 \right]
%%= \zerovec
%\end{equation*}
%%
%The variance of the weights is therefore, similarly using the law of total expectation,
%\begin{equation*}
%\Var[\weights_\nsamples]
%= \E[\AmatPsuedoInv\nubarvec_\nsamples  \nubarvec_\nsamples^\top \AmatPsuedoInvTrans]
%= \sigma^2\E[\AmatPsuedoInv\zbarvec_\nsamples  \zbarvec_\nsamples^\top \AmatPsuedoInvTrans]
%\end{equation*}
%%
%Plugging this into \eqref{eq_chebyshev_ub} completes the proof.
%%Therefore,
%%\begin{align*}
%%\left| \xvec^\top \weights_\nsamples - \xvec^\top \weights^* \right|
%%&\le \frac{\sigma}{\sqrt{\deltaprob}} \sqrt{\xvec^\top \E[\Amat_\nsamples^\inv\zbarvec_\nsamples  \zbarvec_\nsamples^\top \Amat_\nsamples^{-\top}] \xvec}
%%\end{align*}
%%%
%%%
%%\par\vspace{-0.7cm}
\end{proof}

%The general strategy, like policy iteration, is to slowly estimate both the values estimates and the upper-confidence bounds, under a changing policy that acts greedily with respect to the upper-confidence bounds. The above result is for a stationary policy; for such a fixed policy, the above upper-confidence bound can be accurately estimated incrementally. For a slowly changing policy, however, we need to introduce some heuristics to better approximate the upper-confidence, summarized in Appendix \ref{sec_ucls}. We call this algorithm Upper-Confidence Least-Squares (UCLS).\footnote{We do not characterize the regret of UCLS, and instead similarly to policy iteration, rely on a sound update under a fixed policy to motivate incrementally estimating these values as if the policy is fixed and then acting according to them. The only model-free algorithm that achieves a regret bound is RLSVI, but that bound is restricted to the finite horizon, batch, tabular setting. It would be a substantial breakthrough to provide such a regret bound, and is beyond the scope of this work.}  These modifications---to incrementally estimate the upper-confidence bounds for use in control---are largely implementation details, and so we only include them and the complete pseudocode for UCLS in Appendix \ref{sec_ucls}.

\vspace{-0.3cm}
\section{UCLS: Estimating upper-confidence bounds for LSTD in control}\label{sec_ucls}
In this section, we present Upper-Confidence-Least-Squares (UCLS)\footnote{We do not characterize the regret of UCLS, and instead similarly to policy iteration, rely on a sound update under a fixed policy to motivate incrementally estimating these values as if the policy is fixed and then acting according to them. The only model-free algorithm that achieves a regret bound is RLSVI, but that bound is restricted to the finite horizon, batch, tabular setting. It would be a substantial breakthrough to provide such a regret bound, and is beyond the scope of this work.}, a control algorithm, which incrementally estimates the upper-confidence bounds provided in Theorem \ref{thm_main}, for guiding on-policy exploration.
%The upper-confidence bounds derived in Section \ref{sec_UCBound} are sound without requiring i.i.d. assumptions;
The upper-confidence bounds are sound without requiring i.i.d. assumptions;
however, they are derived for a fixed policy. In control, the policy is slowly changing, and so instead we will be slowly tracking this upper bound. The general strategy, like policy iteration, is to slowly estimate both the value estimates and the upper-confidence bounds, under a changing policy that acts greedily with respect to the upper-confidence bounds. Tracking these upper bounds incurs some approximations; we identify and address potential issues here.
The complete psuedocode for UCLS is given in the Appendix (Algorithm \ref{alg_UCLS}).
%The complete psuedocode for UCLS and a more thorough discussion about the steps taken to design solutions for the potential issues in estimating incremental upper-confidence bounds is in Appendix \ref{sec_ucls}.

%First, as we are evaluating the upper-bounds for a changing policy in control, we ensure the policy changes slowly. In order to achieve this in the framework of LSTD, we use exponential moving average, rather than a sample average to estimate the

First, we are not evaluating one fixed policy; rather, the policy is changing. The estimates $\Amat_\nsamples$ and $\bvec_\nsamples$ will therefore be out-of-date. As is common for LSTD with control, we use an exponential moving average, rather than a sample average, to estimate $\Amat_\nsamples$, $\bvec_\nsamples$ and the upper-confidence bound. The exponential moving average uses $\Amat_\nsamples = (1-\beta) \Amat_{\nsamples -1} + \beta \zvec_\nsamples (\xvec_{t} - \gamma \xvec_{t+1})^\top$, for some $\beta \in [0,1]$. If $\beta = 1/\nsamples$, then this reduces to the standard sample average; otherwise, for a fixed $\beta$, such as $\beta = 0.01$, more recent samples have a higher weight in the average.
Because an exponential average is unbiased, the result in Theorem \ref{thm_main} would still hold, and in practice the update will be more effective for the control setting.

Second, we cannot obtain samples of the noise $\nu_t = r_{t+1} + \gamma_{t+1} \xvec_{t+1}^\top \wvec^* - \xvec_t^\top \wvec^*$, which is the TD-error for the optimal value function parameters $\weights^*$ (see Equation \eqref{eq_nu}). Instead, we use $\delta_t$ as a proxy.
This proxy results in an upper bound that is too conservative---too loose---because $\delta_t$ is likely to be larger than $\nu_t$.
This is likely to ensure sufficient exploration, but may cause more exploration than is needed. The moving average update
\begin{equation}
\nubarvec_t = \nubarvec_{t-1} + \beta_t (\delta_t \zvec_t  -  \nubarvec_{t-1})
\end{equation}
should also help mitigate this issue, as older $\delta_t$ are likely larger than more recent ones.

Third, the covariance matrix $\Cmat$ estimating $\E[\Amat_\nsamples^\inv\nubarvec_\nsamples \nubarvec_\nsamples^\top \Amat_\nsamples^\inv]$ could underestimate covariances, depending on a skewed distribution over states and depending on the initialization.
This is particularly true in early learning, where the distribution over states is skewed to be higher near the start state; a sample average can result in underestimates in as yet unvisited parts of the space. To see why, let $\avec = \Amat_\nsamples^\inv\nubarvec_\nsamples$. The covariance estimate $\Cmat_{ij} = \E[\avec_i \avec_j]$ corresponds to feature $i$ and $j$. The agent begins in a certain region of the space, and so features that only become active outside of this region will be zero, providing samples $\avec_i \avec_j = 0$. As a result, the covariance is artificially driven down in unvisited regions of the space, because the covariance accumulates updates of 0. Further, if the initialization to the covariance $\Cmat_{ii}$ is an underestimate, a visited state with high variance will artificially look more optimistic than an unvisited state.

We propose two simple approaches to this issue: updating $\Cmat$ based on locality and adaptively adjusting the initialization to $\Cmat_{ii}$. Each covariance estimate $\Cmat_{ij}$ for features $i$ and $j$ should only be updated if the sampled outer-product is relevant, with the agent in the region where $i$ and $j$ are active.
To reflect this locality, each $\Cmat_{ij}$ is updated with the $\avec_i \avec_j$ only if the eligibility traces is non-zero for $i$ and $j$.
To adaptively update the initialization, the maximum observed $\avec_i^2$ is stored, as $\cmax$, and the initialization $c_0$ to each $\Cmat_{ii}$ is retroactively updated using
\begin{equation*}
\Cmat_{ii} = \Cmat_{ii} - (1-\beta)^{c_i} c_0 + (1-\beta)^{c_i} \cmax
\end{equation*}
where $c_i$ is the number of times $\Cmat_{ii}$ has been updated.
This update is equivalent to having initialized $\Cmat_{ii} = \cmax$. We provide a more stable retroactive update to $\Cmat_{ii}$, in the pseudocode in Algorithm \ref{alg_UCLS}, that is equivalent to this update.

Fourth, to improve the computational complexity of the algorithm, we propose an alternative, incremental strategy for estimating $\weights$, that takes advantage of the fact that we already need to estimate the inverse of $\Amat$ for the upper bound.
%We propose an alternative, incremental strategy that takes advantage of the fact that we are already need to estimate the inverse of $\Amat$ for the upper bound.
% using an estimate For our incremental algorithm, however, we do not need such precise solutions, and in fact because our system is changing gradually, it would be needless to obtain highly accurate solutions on each step.
%Instead, we can take advantage of the summarized information in $\Amat$ to improve the update, but avoid poor conditioning by avoiding directly computing $\Amat^\inv$.
In order to do so, we make use of the summarized information in $\Amat$ to improve the update, but avoid directly computing $\Amat^\inv$ as it may be poorly conditioned.
%To do so,
Instead, we maintain an approximation $\Bmat \approx \Amat^\invt$ that uses a simple gradient descent update, to minimize $\| \Amat^\top \Bmat \xvec_t - \xvec_t \|_2^2$. If $\Bmat$ is the inverse of $\Amat^\top$, then
this loss is zero; otherwise, minimizing it provides an approximate inverse.
This estimate $\Bmat$ is useful for two purposes in the algorithm. First, it is clearly needed to estimate the upper-confidence bound. Second, it also provides a pre-conditioner for the iterative update  $\weights = \weights + \Gmat (\bvec - \Amat \weights)$, for preconditioner $\Gmat$. The optimal preconditioner is in fact the inverse of $\Amat$, if it exists. We use $\Gmat = \Bmat^\top + \eta \eye$ for a small $\eta > 0$ to ensure that the preconditioner is full rank.
Developing this stable update for LSTD required significant empirical investigation into alternatives; in addition to providing a more practical UCLS algorithm, we hope it can improve the use of LSTD in other applications.

\vspace{-0.3cm}
\section{Experiments}

% to using confidence intervals for action selection,

We conducted several experiments to investigate the benefits of UCLS' directed exploration against other methods that use confidence intervals for action selection, to evaluate sensitivity of UCLS's performance with respect to its key parameter $p$, and to contrast the advantage contextual variance estimates offer over global variance estimates in control. Our experiments were intentionally conducted in small---though carefully selected---simulation domains so that we could conduct extensive parameter sweeps, hundreds of runs for averaging, and compare numerous state-of-the-art exploration algorithms (many of which are computationally expensive on larger domains). We believe that such experiments constitute a significant contribution, because effectively using confidence bounds for model free-exploration in RL is still in its infancy---not yet at the large-scale demonstration state--with much work to be done. This point is highlighted nicely below as we demonstrate that several recently proposed exploration methods fail on these simple domains.
% local vs global variance estimates
% varrying confidence level
%
%Redited version MATT
%Our experiments focus on the efficacy of including upper-confidence bounds as a form of optimism for exploration.

\vspace{-0.2cm}
\subsection{Algorithms}

We compare UCLS to DGPQ \citep{grande2014sample}, UCBootstrap \citep{white2010interval}, our extension of LSPI-Rmax to an incremental setting \citep{li2009online} and RLSVI \citep{osband2016generalization}. In-depth descriptions of each algorithm and implementation details can be found in the Appendix. These algorithms are chosen because they either keep confidence intervals explicitly, as in UCBootstrap, or implicitly as in DGPQ and RLSVI. In addition, we included LSPI-Rmax as a natural alternative approach to using LSTD to maintain optimistic value estimates.

We also include Sarsa with $\epsilon$-greedy, with $\epsilon$ optimized over an extensive parameter sweep. Though $\epsilon$-greedy is not a generally practical algorithm, particularly in larger worlds, we include it as a baseline.
We do not include Sarsa with optimistic initialization, because even though it has been a common heuristic, it is not a general strategy for exploration. Optimistic initialization can converge to suboptimal solutions if initial optimism fades too quickly \citep{white2010interval}.
Further, initialization only happens once, at the beginning of learning. If the world changes, then an agent relying on systematic exploration due to its initialization may not react, because it no longer explores. For completeness comparing to previous work using optimistic initialization, we include such results in Appendix \ref{extResults}.

\begin{figure*}[t]
\centering
  \includegraphics[width=0.9\textwidth]{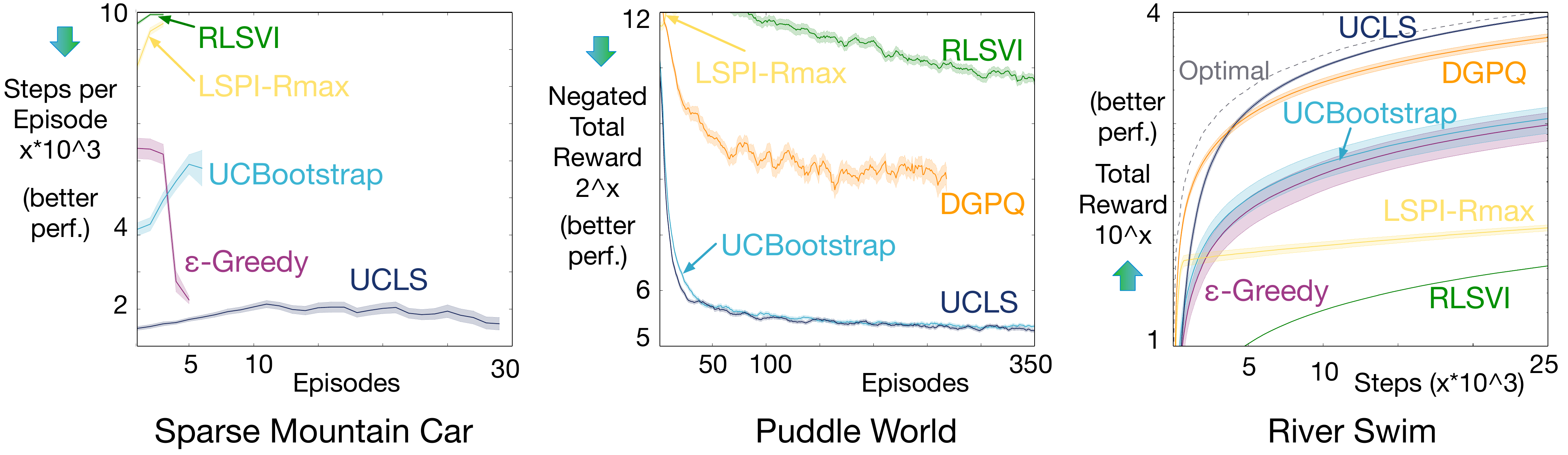}
\caption{A comparison of speed of learning in Sparse Mountain Car, Puddle World and River Swim. In plots (a) and (b) lower on y-axis are better, whereas in (c) curves higher along y-axis are better. Sparse Mountain Car and Puddle World are episodic problems with a fixed experience budget. Thus the length of the lines in plots (a) and (b) indicate how many episodes each algorithm completed over 50,000 steps, and the height on the y-axis indicates the quality of the learned policy---lower indicates better performance. Note RLSVI did not show significant learning after 50,000 steps. The RLSVI result in Puddle World uses a budget of 1 million.}
\vspace{-0.5cm}
\label{fig:LC}
\end{figure*}

%Original
% Our experiments focus on the efficacy of including upper-confidence bounds as a form of optimism for exploration. We compare to DGPQ \citep{grande2014sample}, UCBootstrap \citep{white2010interval}, LSPI-Rmax \citep{li2009online}, and UCLS with global variance. For convenience all the algorithms can be found in the appendix.
% %We compare to DGPQ (see Algorithm \ref{alg_dgpq} in the appendix); UCBootstrap (see Algorithm \ref{alg_bootstrap} in the appendix); LSPI-Rmax (see Algorithm \ref{alg_lspi} in the appendix) and UCLS with global variance (see Algorithm \ref{alg_UCLS-G} in the appendix).
% We experimented with both variants of UCBootstrap - contextual and global. As the global version performed better across domains, we show only the global version here as UCBootstrap. Additionally, please note that Sarsa with optimistic initialization performs well in these domains,
% but because optimistic initialization can be highly sensitive to the learning rate and initial value it is not always a viable strategy, and therefore, is not represented here.
%Note that Sarsa with optimistic initialization performs well in these domains, but is not included as optimistic initialization is not a viable strategy for exploration always \textbf{(Why?)}. The optimistic initialization can be seen as an upper-confidence bound, but one which shrinks much too quickly, and can be highly sensitive to the learning rate \textbf{Change}.

\vspace{-0.2cm}
\subsection{Environments}
% To evaluate the performance consistency, we experiment on multiple benchmark environments whose state space is continuous.

\textbf{Sparse Mountain Car} is a version of classic mountain car problem \citet{sutton1998reinforcement}, only differing in the reward structure. The agent only receives a reward of $+1$ at the goal and $0$ otherwise, and a discounted, episodic $\gamma$ of $0.998$. The start state is sampled from the range $[-0.6,-0.4]$ with velocity zero. This domain is used to highlight  how exploration techniques perform when the reward signal is sparse, and thus initializing the value function to zero is not optimistic.

\textbf{Puddle World} is a continuous state 2-dimensional world with $(x,y) \in [0,1]^2$ with 2 puddles: (1) $[0.45,0.4]$ to $[0.45,0.8]$, and (2) $[0.1, 0.75]$ to $[0.45,0.75]$ - with radius 0.1 and the goal is the region $(x,y) \in ([0.95,1.0],[0.95,1.0])$. The agent receives a reward of $-1 -400*d$ on each time step, where $d$ denotes the distance between the agent's position and the center of the puddle, and an undiscounted, episodic $\gamma$ of $1.0$. The agent can select an action
%(up, left, down, right)
to move $0.05 + \zeta$, $\zeta \sim N(\mu = 0, \sigma = 0.1)$. The agent's initial state is uniformly sampled from $(x,y) \in ([0.1,0.3],[0.45,0.65])$. This domain highlights a common difficulty for traditional exploration methods: high magnitude negative rewards, which often cause the agent to erroneously decrease its value estimates too quickly.% on the optimal path and converge to suboptimal solution.

\textbf{River Swim} is a standard continuing exploration benchmark \cite{szita2008themany} inspired by a fish trying to swim upriver, with high reward (+1) upstream which is difficult to reach and, a lower but still positive reward (+0.005), which is easily reachable downstream. We extended this domain to continuous states in $[0,1]$, with a stochastic displacement of $0.1$ when taking an action up or down, with low-probability of success for up. The starting position is sampled uniformly in $[0,0.1]$, and $\gamma=0.99$.
%displaced $0.1 + \zeta$, $\zeta \sim \mathcal{N}(\mu = 0, \sigma^2 = 0.0001)$.
%---extended to continuous states \citealp{osband2013efficient}---

%We finally test on a gridworld domain, where some parts of the world have high variance. TODO: describe this domain.

\vspace{-0.2cm}
\subsection{Experimental Setup}

%Notes: I removed the consistant seeding because I don't think this matters as we do lots of runs. Also we did not do this for DGPQ.
%We also ensure consistent seeding to promote similar stochastic behavior across methods and runs.
We investigate a learning regime where the agents are allowed a fixed budget of interaction steps with the environment, rather than allowing a finite number of episodes of unlimited length. Our primary concern is early learning performance,
thus each experiment is restricted to 50,000 steps, with an episode cutoff (in Sparse Mountain Car and Puddle World) at 10,000 steps. In this regime, an agent that spends a significant time exploring the world during the first episode may not be able to complete many episodes, the cutoff makes exploration easier given the strict budget on experience. Whereas, in the more common framework of allowing a fixed number of episodes, an agent can consume many steps during the first few episodes exploring, which is difficult to detect in the final performance results. We average over 100 runs in River Swim and 200 runs for the other domains
%(more runs to ensure statistical significance)
. For all the algorithms that utilize eligibility traces we set $\lambda$ to be 0.9. For algorithms which use exponential averaging, $\beta$ is set to 0.001, and the regularizer $\eta$ is set to be 0.0001. The parameters for UCLS are fixed. RLSVI's weights are recalculated using all experienced transitions at the beginning of an episode in Puddle World and Sparse Mountain Car, and every 5,000 steps in River Swim. The parameters of competitors, where necessary, are selected as the best from a large parameter sweep.

%To encourage exploration, the initial estimated global variance for UCLS is set to 100. For both the contextual and global variants of UCLS we show experiments for two values of $p$: (1) 0.01 - a confidence value which promotes extensive exploration, and (2) 0.5 - a  confidence value which promotes crude exploration. The parameters of the competitors, where necessary, are selected as the best from a large parameter sweep.

%It is easy to see that this initialization gets exponentially decayed as the global variance estimate converges to the actual on-policy estimate.

All the algorithms except DGPQ use the same representation: (1) Sparse Mountain Car - 8 tilings of 8x8, hashed to a memory space of 512, (2) River Swim - 4 tilings of granularity 32, hashed to a memory space of 128, and (3) Puddle World - 5 tilings of granularity 5x5, hashed to a memory space of 128. DGPQ uses its own kernel-based representation with normalized state information.
%For DGPQ we normalized the domains to have dimensions $[0,1]^d$.

\vspace{-0.2cm}
\subsection{Results \& Analysis}
Our first experiment simply compares UCLS against other control algorithms in all the domains.
%UCBootstrap, RLSVI, DGPQ.
Figure \ref{fig:LC} shows the early learning results across all three domains. In all three domains UCLS achieves the best final performance.
%The plots presented show early learning performance, not always including the entire budget of experience. For example in Puddle World, UCLS goes on to complete about 1200 episodes (within 50,000 steps), and UCLS and DGPQ get close to the optimal policy in Riversim after 50,000 steps. % (a regret of 1000 compared with the optimal policy).
In Sparse Mountain Car, UCLS learns faster than the other methods, while in River Swim DGPQ learns faster initially. UCBootstrap and UCLS learn at a similar rate in Puddle World, which is a cost-to-goal domain. UCBootstrap, and bootstrapping approaches generally, can suffer from insufficient optimism, as they rely on sufficiently optimistic or diverse initialization strategies \citep{white2010interval,osband2016deep}. LSPI-Rmax and RLSVI do not perform well in any of the domains. DGPQ does not perform as well as UCLS in Puddle World, and exhibits high variance compared with the other methods. In Puddle World, UCLS goes on to finish 1200 episodes in the alloted budget of steps, whereas in River Swim both UCLS and DGPQ get close to the optimal policy by the end of the experiment.

%Results comparing UCLS to the aformentioned competitors can be seen in Figure \ref{fig:LC}. UCLS performs consistently well across all settings. LSPI-Rmax struggles to do well in these domains. This could be attributed to the algorithm exploring more in bad parts of the domain (such as the puddles in Puddle World), or in how the algorithm decides when a state is known. DGPQ does exceptionally well in River Swim (Figure \ref{fig:LC}c), but we remind the reader that

The DGPQ algorithm uses the maximum reward (Rmax) to initialize the Gaussian processes. In Sparse Mountain Car this effectively converts the problem back into the traditional -1 per-step formulation. In this traditional variant of Mountain Car UCLS significantly outperforms DGPQ (Appendix \ref{extResults}).
Sarsa with $\epsilon$-greedy learns well in Puddle world as it is a cost-to-goal problem in which by default Sarsa uses optimistic initialization, and therefore is reported in the Appendix.
.

\begin{SCfigure}
  \caption{The effect of the confidence parameter $p$ on the policy, in River Swim, using context-dependent variance (UCLS) and global variance (GV-UCB). The values for $p$ are $\{10^{-5}, [1,2,\dots,9]\times10^{-3},10^{-2},10^{-1}\}$. }
  \includegraphics[width=0.65\textwidth,scale=0.25]{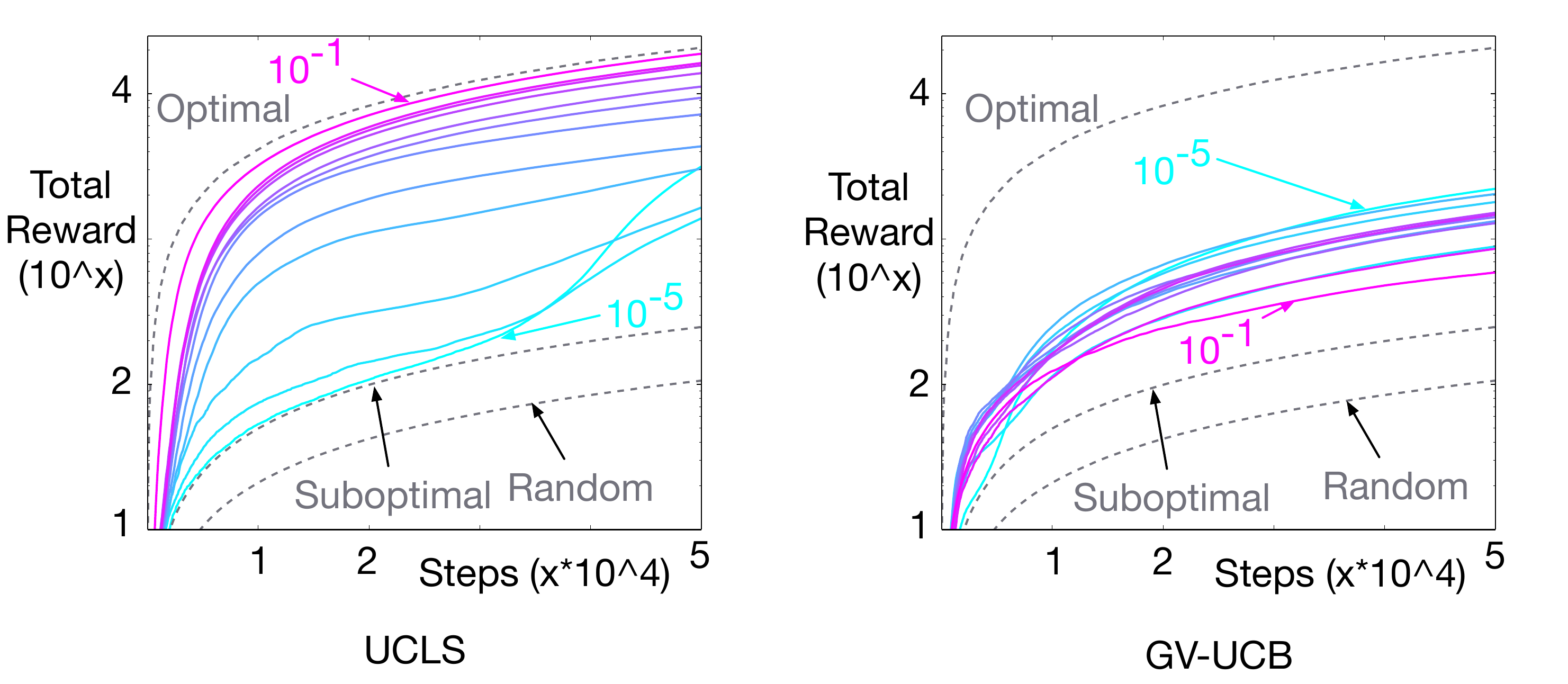}
  \label{fig:RS-pSen}
  \vspace{-0.5cm}
\end{SCfigure}

% Additionally, the plots presented show early learning performance, not always including the entire budget of experience. For example in Puddle World, UCLS goes on to complete about 1200 episodes (within 50,000 steps), and UCLS and DGPQ get close to the optimal policy in Riversim after 50,000 steps. % (a regret of 1000 compared with the optimal policy).

Next we investigated the impact of the confidence level $1-p$, on the performance of UCLS in River Swim.
The confidence interval radius is proportional to $\sqrt{1+1/p}$; smaller $p$ should correspond to a higher rate of exploration.
%results in more exploration,  and can be viewed as controlling the period of exploration before the agent exploits the learned dynamics.
%As the parameter shrinks the confidence bounds grow encouraging more exploration before a more exploitative policy is enacted, as can be seen from the UCLS plot in Figure \ref{fig:RS-pSen}.
%keep this?
%MATHA edit paragraph below:
In Figure \ref{fig:RS-pSen}, smaller $p$ resulted in a slower convergence rate, but all values eventually reach the optimal policy.
% Even with a $p$ that more carefully---and exhaustively---explores the space, the confidence interval radius does shrink and the agent eventually reaches optimal.
%TODO: As a minor point of interest, amongst the smaller $p$, there is a visible crossover in the plot. This suggests that exploring more in early learning promoted reaching the optimal policy more quickly, with a sudden increase in performance.
%
%TODO Interestingly, as shown in in Figure \ref{fig:RS-pSen}, exploring more in early learning, with smaller $p$, seems to result in exploitative policies sooner, as can be inferred from the visible crossover in the plot.
%This is also advantageous as compared to a random exploration strategy, because the confidence bounds guide the agent to explore more in states that are less known. In the sensitivity plot we also see that even with more exploring the agent eventually will settle on a good policy outperforming the sub-optimal and random policies.

Finally, we investigate the benefit using contextual variance estimates over global variance estimates within UCLS. In Figure  \ref{fig:RS-pSen}, we also show the effect of various $p$ values on the performance of the algorithm resulting from Corollary \ref{cor_global}, which we call Global Variance-UCB (GV-UCB) (see Appendix \ref{sec_gvucb} for more details about this algorithm).
%pertaining to the algorithm and the updates used are discussed
For this range of $p$, UCLS still converges to the optimal policy, albeit at different rates. Using a global variance estimates (GV-UCB), on the other hand, results in significant over-estimates of variance, resulting in poor performance.
% Interestingly, GV-UCB with the best value of $p$ performs similarly to UCLS with its worst performing $p$ value. Small values $p$ perform worse for GV-UCB, whereas UCLS exhibits the opposite relationship.

%value on the regret experienced by the control policy flips with GV-UCB. %TODO This is likely because a global variance provided more severe over-estimates of the upper-confidence bound, and so a smaller $p$ was more effective because it scaled down these over-estimates.

\vspace{-0.3cm}
\section{Conclusion and Discussion}
 This paper develops a sound upper-confidence bound on the value estimates for least-squares temporal difference learning (LSTD),
without making i.i.d. assumptions about noise distributions. In particular, we allow for context-dependent noise, where variability
could be due to noise in rewards, transition dynamics or even limitations of the function approximator. We then introduce an algorithm, called UCLS,
that estimates these upper-confidence bounds incrementally, for policy iteration. We demonstrate empirically that UCLS requires far fewer exploration steps to find high-quality policies compared to several baselines, across domains chosen to highlight different exploration difficulties.
%and otherwise explores sufficiently in all the domains to reach the optimal policy.
%succeeds on River Swim, which causes many algorithms to under-explore.

The goal of this paper is to provide an incremental, model-free, data-efficient, directed exploration strategy.
%The method was indeed data-efficient; however, more investigation is needed to ensure that using these bounds for exploration ensures stochastic optimism, and proving convergence to optimal policies.
The upper-confidence bounds for action-values for fixed policies are one of the few available under function approximation, and so a step towards exploration with optimistic values in the general case. A next step is to theoretically show that using these upper bounds for exploration ensures stochastic optimism, and so converges to optimal policies.

% TODO: Maybe add comment about Eluder dimension above?
%This paper provides sound confidence intervals for a fixed policy; an important next step is to analyze the confidence intervals
%for the non-stationary case, where the policy is slowly changing.
%We do not provide PAC-MDP bounds in this work, because our goal is to begin shifting towards continuous state, infinite-horizon problems.
%One promising direction for analysis of this algorithm is with the Eluder dimension \citep{russo2013eluder}, which explicitly avoids the unrealistic assumptions in PAC-MDP results. It remains an open question, though, on how we should analyze the sample complexity and optimality of these exploration algorithms
%for the continuous-state setting. Such an framework is outside the scope of this paper, but we hope by providing a practical algorithm,
%we can contribute to the motivation to more concretely obtain such a formalism.

One promising aspect of UCLS is that it uses least-squares to efficiently summarize past experience,
but is not tied to a specific state representation.
%The Gaussian process regression approaches, on the other hand, require the use of Gaussian processes and so can be onerous to implement for those not experienced with Gaussian processes.
Though we considered a fixed representation for UCLS, it is feasible that an analysis for the non-stationary case could be used as well for the setting where
the representation is being adapted over time. If the representation drifts slowly, then UCLS may be able to similarly track the upper-confidence bounds.
Recent work has shown that combining deep Q-learning with Least-squares can result in significant performance gains over vanilla DQN\citep{levine2017shallow}. We expect that combining deep networks and UCLS could result in even larger gains, and is a natural direction for future work. %A future direction is to investigate the utility of UCLS within neural networks.
% where m learning algorithms including DQN use $\epsilon$-greedy and unsupervised count-based approaches.
%is common,
%which is not a data-efficient exploration strategy.

\vspace{-0.3cm}
\section{Acknowledgements}
We would like to thank Bernardo \'Avila Pires and Jian Qian for their helpful comments, alongwith Calcul Qu\'ebec (\url{www.calculquebec.ca}) and Compute Canada (\url{www.computecanada.ca}) for the computing resources used in this work.
% We would like to thank Bernardo \'Avila Pires and Jian Qian for their helpful comments. This research was enabled in part due to computational support provided by Calcul Qu\'ebec (\url{www.calculquebec.ca}) and Compute Canada (\url{www.computecanada.ca}).
% This research was enabled in part due to support provided by Calcul Qu\'ebec and Compute Canada.
% This research was enabled in part due to support provided by Calcul Qu\'ebec \url{www.calculquebec.ca} and Compute Canada \url{www.computecanada.ca}.

\small
\bibliographystyle{abbrvnat}
\bibliography{paper.bib}

\begin{thebibliography}{47}
\providecommand{\natexlab}[1]{#1}
\providecommand{\url}[1]{\texttt{#1}}
\expandafter\ifx\csname urlstyle\endcsname\relax
  \providecommand{\doi}[1]{doi: #1}\else
  \providecommand{\doi}{doi: \begingroup \urlstyle{rm}\Url}\fi

\bibitem[Abbasi-Yadkori and Szepesvari(2014)]{abbasiyadkori2014bayesian}
Y.~Abbasi-Yadkori and C.~Szepesvari.
\newblock {Bayesian Optimal Control of Smoothly Parameterized Systems: The Lazy
  Posterior Sampling Algorithm}.
\newblock In \emph{Uncertainty in Artificial Intelligence}, 2014.

\bibitem[Auer and Ortner(2006)]{auer2006logarithmic}
P.~Auer and R.~Ortner.
\newblock {Logarithmic Online Regret Bounds for Undiscounted Reinforcement
  Learning.}
\newblock \emph{Advances in Neural Information Processing Systems}, 2006.

\bibitem[Bartlett and Tewari(2009)]{bartlett2009regal}
P.~L. Bartlett and A.~Tewari.
\newblock {REGAL - A Regularization based Algorithm for Reinforcement Learning
  in Weakly Communicating MDPs.}
\newblock In \emph{Conference on Uncertainty in Artificial Intelligence}, 2009.

\bibitem[Boyan(2002)]{boyan2002technical}
J.~A. Boyan.
\newblock Technical update: Least-squares temporal difference learning.
\newblock \emph{Machine learning}, 49\penalty0 (2-3):\penalty0 233--246, 2002.

\bibitem[Bradtke and Barto(1996)]{bradtke1996linear}
S.~J. Bradtke and A.~G. Barto.
\newblock Linear least-squares algorithms for temporal difference learning.
\newblock \emph{Machine learning}, 22\penalty0 (1-3):\penalty0 33--57, 1996.

\bibitem[Brafman and Tennenholtz(2003)]{brafman2003rmax}
R.~Brafman and M.~Tennenholtz.
\newblock {R-max-a general polynomial time algorithm for near-optimal
  reinforcement learning}.
\newblock \emph{The Journal of Machine Learning Research}, 2003.

\bibitem[Chu et~al.(2011)Chu, Li, Reyzin, and Schapire]{chu2011contextual}
W.~Chu, L.~Li, L.~Reyzin, and R.~E. Schapire.
\newblock {Contextual Bandits with Linear Payoff Functions.}
\newblock In \emph{International Conference on Artificial Intelligence and
  Statistics}, 2011.

\bibitem[Grande et~al.(2014)Grande, Walsh, and How]{grande2014sample}
R.~Grande, T.~Walsh, and J.~How.
\newblock {Sample Efficient Reinforcement Learning with Gaussian Processes}.
\newblock \emph{International Conference on Machine Learning}, 2014.

\bibitem[Jaksch et~al.(2010)Jaksch, Ortner, and Auer]{jaksch2010near}
T.~Jaksch, R.~Ortner, and P.~Auer.
\newblock {Near-optimal Regret Bounds for Reinforcement Learning}.
\newblock \emph{The Journal of Machine Learning Research}, 2010.

\bibitem[Jong and Stone(2007)]{jong2007model}
N.~Jong and P.~Stone.
\newblock {Model-based exploration in continuous state spaces}.
\newblock \emph{Abstraction, Reformulation, and Approximation}, 2007.

\bibitem[Jung and Stone(2010)]{jung2010gaussian}
T.~Jung and P.~Stone.
\newblock {Gaussian processes for sample efficient reinforcement learning with
  RMAX-like exploration}.
\newblock In \emph{Machine Learning: ECML PKDD}, 2010.

\bibitem[Kaelbling(1993)]{kaelbling1993learning}
L.~P. Kaelbling.
\newblock \emph{{Learning in embedded systems}}.
\newblock MIT press, 1993.

\bibitem[Kaelbling et~al.(1996)Kaelbling, Littman, and
  Moore]{kaelbling1996reinforcement}
L.~P. Kaelbling, M.~L. Littman, and A.~W. Moore.
\newblock {Reinforcement learning: A survey}.
\newblock \emph{Journal of Artificial Intelligence Research}, 1996.

\bibitem[Kakade et~al.(2003)Kakade, Kearns, and
  Langford]{kakade2003exploration}
S.~Kakade, M.~Kearns, and J.~Langford.
\newblock {Exploration in metric state spaces}.
\newblock In \emph{International Conference on Machine Learning}, 2003.

\bibitem[Kawaguchi(2016)]{kawaguchiAAAI2016}
K.~Kawaguchi.
\newblock {Bounded Optimal Exploration in MDP}.
\newblock In \emph{AAAI Conference on Artificial Intelligence}, 2016.

\bibitem[Kearns and Singh(2002)]{kearns2002near}
M.~J. Kearns and S.~P. Singh.
\newblock {Near-Optimal Reinforcement Learning in Polynomial Time.}
\newblock \emph{Machine Learning}, 2002.

\bibitem[Lagoudakis and Parr(2003)]{lagoudakis2003least}
M.~G. Lagoudakis and R.~Parr.
\newblock {Least-squares policy iteration}.
\newblock \emph{The Journal of Machine Learning Research}, 2003.

\bibitem[Levine et~al.(2017)Levine, Zahavy, Mankowitz, Tamar, and
  Mannor]{levine2017shallow}
N.~Levine, T.~Zahavy, D.~J. Mankowitz, A.~Tamar, and S.~Mannor.
\newblock Shallow updates for deep reinforcement learning.
\newblock In \emph{Advances in Neural Information Processing Systems}, pages
  3138--3148, 2017.

\bibitem[Li et~al.(2009)Li, Littman, and Mansley]{li2009online}
L.~Li, M.~Littman, and C.~Mansley.
\newblock {Online exploration in least-squares policy iteration}.
\newblock In \emph{International Conference on Autonomous Agents and Multiagent
  Systems}, 2009.

\bibitem[Li et~al.(2010)Li, Chu, Langford, and Schapire]{li2010acontextual}
L.~Li, W.~Chu, J.~Langford, and R.~E. Schapire.
\newblock {A contextual-bandit approach to personalized news article
  recommendation}.
\newblock In \emph{World Wide Web Conference}, 2010.

\bibitem[Martin et~al.(2017)Martin, Sasikumar, Everitt, and
  Hutter]{martin2017count}
J.~Martin, S.~N. Sasikumar, T.~Everitt, and M.~Hutter.
\newblock {Count-Based Exploration in Feature Space for Reinforcement
  Learning}.
\newblock In \emph{International Joint Conference on Artificial IntelligenceI},
  2017.

\bibitem[Meuleau and Bourgine(1999)]{meuleau1999exploration}
N.~Meuleau and P.~Bourgine.
\newblock {Exploration of Multi-State Environments - Local Measures and
  Back-Propagation of Uncertainty.}
\newblock \emph{Machine Learning}, 1999.

\bibitem[Meyer(1973)]{meyer1973generalized}
C.~D. Meyer, Jr.
\newblock Generalized inversion of modified matrices.
\newblock \emph{SIAM Journal on Applied Mathematics}, 24\penalty0 (3):\penalty0
  315--323, 1973.

\bibitem[Miller(1981)]{miller1981inverse}
K.~S. Miller.
\newblock On the inverse of the sum of matrices.
\newblock \emph{Mathematics magazine}, 54\penalty0 (2):\penalty0 67--72, 1981.

\bibitem[Moerland et~al.(2017)Moerland, Broekens, and
  Jonker]{moerland2017efficient}
T.~M. Moerland, J.~Broekens, and C.~M. Jonker.
\newblock {Efficient exploration with Double Uncertain Value Networks}.
\newblock In \emph{Advances in Neural Information Processing Systems}, 2017.

\bibitem[Nouri and Littman(2009)]{nouri2009multi}
A.~Nouri and M.~L. Littman.
\newblock {Multi-resolution Exploration in Continuous Spaces}.
\newblock In \emph{Advances in Neural Information Processing Systems}, 2009.

\bibitem[Ortner and Ryabko(2012)]{ortner2012online}
R.~Ortner and D.~Ryabko.
\newblock {Online Regret Bounds for Undiscounted Continuous Reinforcement
  Learning}.
\newblock In \emph{Advances in Neural Information Processing Systems}, 2012.

\bibitem[Osband and Van~Roy(2017)]{osband2017why}
I.~Osband and B.~Van~Roy.
\newblock {Why is Posterior Sampling Better than Optimism for Reinforcement
  Learning?}
\newblock In \emph{International Conference on Machine Learning}, 2017.

\bibitem[Osband et~al.(2013)Osband, Russo, and Van~Roy]{osband2013efficient}
I.~Osband, D.~Russo, and B.~Van~Roy.
\newblock {(More) Efficient Reinforcement Learning via Posterior Sampling}.
\newblock In \emph{Advances in Neural Information Processing Systems}, 2013.

\bibitem[Osband et~al.(2016{\natexlab{a}})Osband, Blundell, Pritzel, and
  Van~Roy]{osband2016deep}
I.~Osband, C.~Blundell, A.~Pritzel, and B.~Van~Roy.
\newblock {Deep Exploration via Bootstrapped DQN.}
\newblock In \emph{Advances in Neural Information Processing Systems},
  2016{\natexlab{a}}.

\bibitem[Osband et~al.(2016{\natexlab{b}})Osband, Van~Roy, and
  Wen]{osband2016generalization}
I.~Osband, B.~Van~Roy, and Z.~Wen.
\newblock {Generalization and Exploration via Randomized Value Functions.}
\newblock In \emph{International Conference on Machine Learning},
  2016{\natexlab{b}}.

\bibitem[Ostrovski et~al.(2017)Ostrovski, Bellemare, van~den Oord, and
  Munos]{ostrovski2017count}
G.~Ostrovski, M.~G. Bellemare, A.~van~den Oord, and R.~Munos.
\newblock {Count-Based Exploration with Neural Density Models.}
\newblock In \emph{International Conference on Machine Learning}, 2017.

\bibitem[Pazis and Parr(2013)]{pazis2013pac}
J.~Pazis and R.~Parr.
\newblock {PAC optimal exploration in continuous space Markov decision
  processes}.
\newblock In \emph{AAAI Conference on Artificial Intelligence}, 2013.

\bibitem[Plappert et~al.(2017)Plappert, Houthooft, Dhariwal, Sidor, Chen, Chen,
  Asfour, Abbeel, and Andrychowicz]{plappert2017parameter}
M.~Plappert, R.~Houthooft, P.~Dhariwal, S.~Sidor, R.~Y. Chen, X.~Chen,
  T.~Asfour, P.~Abbeel, and M.~Andrychowicz.
\newblock {Parameter Space Noise for Exploration}.
\newblock \emph{arXiv.org}, 2017.

\bibitem[Singh et~al.(2000)Singh, Jaakkola, Littman, and
  Szepesvari]{singh2000convergence}
S.~P. Singh, T.~S. Jaakkola, M.~L. Littman, and C.~Szepesvari.
\newblock {Convergence Results for Single-Step On-Policy Reinforcement-Learning
  Algorithms.}
\newblock \emph{Machine Learning}, 2000.

\bibitem[Strehl and Littman(2004)]{strehl2004exploration}
A.~Strehl and M.~Littman.
\newblock {Exploration via model based interval estimation}.
\newblock In \emph{International Conference on Machine Learning}, 2004.

\bibitem[Strehl et~al.(2006)Strehl, Li, Wiewiora, Langford, and
  Littman]{strehl2006pac}
A.~L. Strehl, L.~Li, E.~Wiewiora, J.~Langford, and M.~L. Littman.
\newblock {PAC model-free reinforcement learning}.
\newblock In \emph{International Conference on Machine Learning}, 2006.

\bibitem[Sutton et~al.(2008)Sutton, Szepesv{\'a}ri, Geramifard, and
  Bowling]{sutton2008dyna}
R.~Sutton, C.~Szepesv{\'a}ri, A.~Geramifard, and M.~Bowling.
\newblock {Dyna-style planning with linear function approximation and
  prioritized sweeping}.
\newblock In \emph{Conference on Uncertainty in Artificial Intelligence}, 2008.

\bibitem[Sutton(1988)]{sutton1988learning}
R.~S. Sutton.
\newblock {Learning to predict by the methods of temporal differences}.
\newblock \emph{Machine Learning}, 1988.

\bibitem[Sutton and Barto(1998)]{sutton1998reinforcement}
R.~S. Sutton and A.~G. Barto.
\newblock \emph{Reinforcement learning: An introduction}.
\newblock MIT press Cambridge, 1998.

\bibitem[Szepesvari(2010)]{szepesvari2010algorithms}
C.~Szepesvari.
\newblock \emph{{Algorithms for Reinforcement Learning}}.
\newblock Morgan {\&} Claypool Publishers, 2010.

\bibitem[Szita and Lorincz(2008)]{szita2008themany}
I.~Szita and A.~Lorincz.
\newblock {The many faces of optimism}.
\newblock In \emph{International Conference on Machine Learning}, 2008.

\bibitem[Szita and Szepesvari(2010)]{szita2010model}
I.~Szita and C.~Szepesvari.
\newblock {Model-based reinforcement learning with nearly tight exploration
  complexity bounds}.
\newblock In \emph{International Conference on Machine Learning}, 2010.

\bibitem[van Seijen and Sutton(2015)]{vanseijen2015adeeper}
H.~van Seijen and R.~Sutton.
\newblock {A deeper look at planning as learning from replay}.
\newblock In \emph{International Conference on Machine Learning}, 2015.

\bibitem[White(2017)]{white2017unifying}
M.~White.
\newblock {Unifying task specification in reinforcement learning}.
\newblock In \emph{International Conference on Machine Learning}, 2017.

\bibitem[White and White(2010)]{white2010interval}
M.~White and A.~White.
\newblock {Interval estimation for reinforcement-learning algorithms in
  continuous-state domains}.
\newblock In \emph{Advances in Neural Information Processing Systems}, 2010.

\bibitem[Wiering and Schmidhuber(1998)]{wiering1998efficient}
M.~A. Wiering and J.~Schmidhuber.
\newblock {Efficient Model-Based Exploration}.
\newblock In \emph{Simulation of Adaptive Behavior From Animals to Animats},
  1998.

\end{thebibliography}

\normalsize
\newpage

\appendix
%!TEX root = paper.tex

%\input{related.tex}

% Remvoed old results, which just had iffalse anyway
%\input{old_results}

\begin{figure*}[t]
\includegraphics[width=\textwidth]{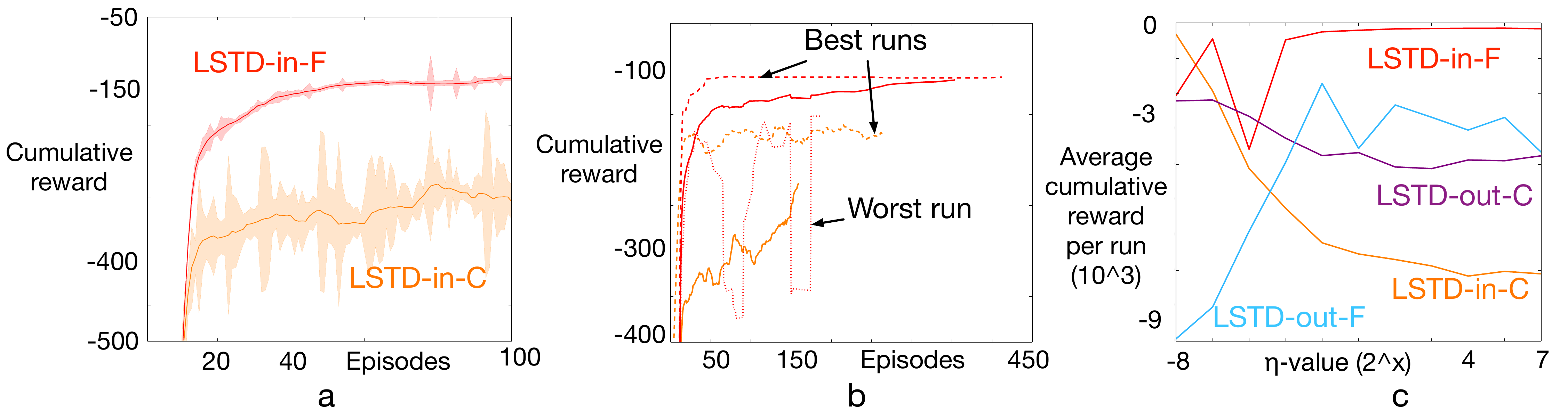}
\caption{Learning performance in Mountain Car for LSTD-in and LSTD-out with $\eta$ kept constant through learning (-C) and $\eta$ fading with time (-F). (a) Early learning curves for LSTD-in. This plot does not include LSTD-out as it performed too poorly to be visible. (b) Learning curves for LSTD-in with best and worst runs. LSTD-in-C's worst run performed too poorly to be visible. (c) Parameter sensitivity for both variants LSTD-in and LSTD-out to $\eta$/$\eta_r$. }
\label{fig:MC-CG}
\end{figure*}

\section{Issues with LSTD for control}\label{app_issues}

%Matts version with more edits.
LSTD is a more data-efficient algorithm than its incremental counterpart TD, and typically performs quite well in policy evaluation. This is primarily due to TD only using each sample once for a stochastic update with a tuned stepsize parameter. In the case of control, LSTD performs surprisingly well without $\epsilon$-greedy exploration and lack of an optimism strategy. We highlight here the inadvertent use of the regularization parameter as a form of optimism for LSTD  in control, and empirically show when this strategy fails leading us to UCLS as a sound approach in using LSTD in control.
%Matts version with edits: Notes: Mostly moved the discussion of the use of \eta in exploration below.
% LSTD should be more data-efficient than the original TD update as TD only uses each sample once for a stochastic update. TD also has a tuned stepsize parameter, which is hard to tune. On the flip-side, obtaining the LSTD solution corresponds to solving a linear system, which could be ill-conditioned. We highlight here why LSTD performs surprisingly well without $\epsilon$-greedy exploration and optimistic initialization, and show when the regularization parameter $\eta$ cannot be used to encourage exploration.

In practice, the inverted matrix $\Amat^\inv$ is often directly maintained using a Sherman-Morrison update, with a small regularizer $\eta$ added to the matrix $\Amat$ to guarantee invertibility \citep{szepesvari2010algorithms}.
% MARTHAC: I removed this, since we dont actually show SM
%Further, to enable the update to better track a changing control policy, we use an exponentially weighted average with parameter $\beta_t$. For each update to $\Amat$ and $\bvec$, more recent samples are more highly weighted with an exponential average, namely $\bvec = (1-\beta_t) \bvec + r \zvec$ for some $\beta_t \in (0,1)$, such as $\beta_t = 0.001$.
%This standard update is provided in Algorithm \ref{alg_LSTD_SM}.

% The regularization parameter is

%MATT moved to below.
% To better control the regularization parameter $\eta$, we use a conjugate gradient update rather than the more common Sherman-Morrison (SM) update.
% This is because inherently, the SM update reduces $\eta$ as $\eta/t$ for time step $t$, making it difficult to set this parameter. For example, $\Amat_t$ may become non-invertible after $t = 1000$ steps, and require a large $\eta$ to start to ensure that $\eta/t$ is invertible. Instead, with a conjugate gradient update,
% we can more easily allow $\eta$ to stay a small constant value or allow it to fade away.

\newcommand{\etat}{\eta_r}
\newcommand{\etasm}{\eta}

There are two objectives that can be solved when dealing with an ill-conditioned system $\Amat \wvec = \bvec$. The most common is to use Tikohonov regularization solving, referred to here as \textit{LSTD-out}.
\begin{align*}
\min_\weights \| \Amat \weights - \bvec \|_2^2 + \etat \| \weights \|_2^2
\end{align*}
%
% The initialization for the SM update causes the algorithm to instead solve for $(\Amat + \etasm \eye) \weights = \bvec$, i.e., solving
Another approach is to solve the system
\begin{align*}
\min_\weights \| (\Amat + \etasm \eye) \weights - \bvec \|_2^2
\end{align*}
The second approach is implicitly what is solved when a Sherman-Morrison update is used for $\Amat^\inv$, with a small regularizer $\eta$ added to the matrix $\Amat$ to guarantee invertibility.  This approach is referred to here as \textit{LSTD-in}.
When $\eta = 0$, both approaches are solving $\| \Amat \weights - \bvec \|_2^2$, which may have infinitely many solutions if $\Amat$ is not full rank.
While the Tikohonov regularization strategy is more common, the second approach is useful for enabling use of the incremental Sherman-Morrison update to facilitate maintaining $\Amat^\inv$ directly.

Another choice in regularizing the ill-conditioned system is in how $\eta$ decays over time. A small fixed $\eta$ can be used as a constant regularizer, even as the number of samples increases, because the true $\Amat$ may be ill-conditioned. However, more regularization could also be used at the beginning and then decayed over time. The incremental Sherman-Morrison update implicitly decays $\eta$ proportionally to $\frac{1}{t}$.
% Otherwise, Tikohonov regularization is the more common strategy to regularize linear systems, but the second is necessary to use an incremental Sherman-Morrison update.
%We therefore test both, to be more agnostic to this choice.

\newcommand{\tikzsize}{190pt}
\newcommand{\tikzscale}{0.9}

%\begin{figure}
%\begin{minipage}{0.48\textwidth}
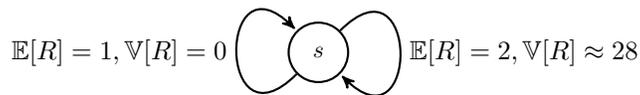
\begin{figure}[h]
\centering
\begin{tikzpicture}[->, >=stealth', auto, semithick, node distance=2cm]
\tikzstyle{every state}=[fill=white,draw=black,thick,text=black,scale=\tikzscale]
\node[state]    (A)                     {$s$};
%\node[state]    (B)[right of=A]   {$s_2$};
%\node[state]    (C)[right of=B]   {$s_3$};
\Loop[dist=1.5cm,dir=EA,label=$\text{$\E[R] = 2, \Var[R] \approx 28$}$,labelstyle=right](A)
\Loop[dist=1.5cm,dir=WE,label=$\text{$\E[R] = 1, \Var[R] = 0$}$,labelstyle=left](A)
%\path
%(A) edge[loop left = 30]     node{left, $r = 1$}         (A)
%(A) edge[loop right]     node{right, $\E[R] = 2, \Var[R] = 30$}         (A);
%    edge[bend left]     node{}     (B)
%(B) edge[bend left = 30]       node{}           (C)
 %     edge[bend right = -30]     node{}      (A)
%(C) edge[bend right = -30]       node{$\gamma(s_3,r,s_1) = 0$}           (A);
\end{tikzpicture}
\caption{One-state world, where the optimal action (right) has high-variance; the reward here is uniformly sampled from within the set $\{-5,-2,2,5,10\}$. LSTD, with $\epsilon = 0$ and $\eta$ large, fails in this world, unlike the cost-to-goal problems. }\label{fig_onestate}
%\vspace{-2cm}
\end{figure}

We conducted an empirical study using LSTD without an $\epsilon$-greedy exploration strategy in two domains: Mountain Car and a new One-State world.
%Mountain Car is chosen as a well-understood cost-to-goal problem, where LSTD has previously been tested.
One-State world---depicted in Figure \ref{fig_onestate}---simulates a typical setting where sufficient exploration is needed: one outcome with low variance and lower expected value and one outcome with high variance and higher expected value. For an algorithm that does not explore sufficiently, it is likely to settle on the suboptimal action, but more immediately rewarding low-variance outcome. This world simulates a larger continuous navigation task from \citealp{white2010interval}. We include results for both systems described above and consider a fading version (shown by \textit{-F}) or a constant regularization parameter (shown by \textit{-C}).

% We include results for the two regularization approaches---using $\etat$ or $\etasm$---and by considering a fading or non-fading regularization parameter over time.

Figure \ref{fig:MC-CG} shows results for the four different LSTD strategies in Mountain Car. The Tikohonov regularization, with $\etat$, is unable to learn an optimal policy in this domain, whereas with either constant or fading $\etasm$, the agent can learn an optimal policy. This is surprising, considering we use neither randomized exploration nor optimistic initialization. The parameter sensitivity curve, shown in plot c, indicates $\etasm$ and $\etat$ needs to be sufficiently large as time passes in order to find an optimal policy.

Next, we show that neither regularization strategy with fading $\eta$ is effective in the One-State world. The optimal strategy is to take the Right action, to get an expected reward of $2$ under a higher variance for obtaining rewards. All of the LSTD variants fail for this domain, because $\eta$ no longer plays a role in encouraging exploration. To verify that a directed exploration strategy helps, we experiment with $\epsilon$-greedy exploration, with $\epsilon=0.1$, decayed by a factor of $0.2$ every 100 steps (shown in Figure \ref{fig:1State-CG-etaS}). With $\epsilon$-greedy,  and small values of $\etat$ and $\etasm$, the policy converges to the optimal action, whereas it fails to with higher values of $\etat$ and $\eta$.

These results suggest that $\eta$'s role in exploration has obscured our understanding of how to use LSTD for control.
LSTD, with sufficient optimism does seem to reach optimal solutions, and unlike \citet{sutton2008dyna},
we did not find any issues with forgetting. This further explains why there have been previous results with small $\epsilon$ for LSTD in cost-to-goal problems, that nonetheless still obtained the optimal policy \citep{vanseijen2015adeeper}. Therefore, in developing UCLS, we more explicitly add optimism to LSTD, and ensure $\eta$ is strictly used as a regularization parameter (to ensure well-conditioned updates).

%\vspace{-2cm}

\begin{SCfigure}
  \caption{$\eta$-sensitivity in 1-State world with various LSTD updates. Sarsa with optimistic initialization $\alpha = 0.001$ is used as a baseline.
	The y-axis represents percentage optimal behaviour, where optimal behaviour is choosing to go right, in 20k steps (averaged over 30 runs).
	Sarsa with optimistic initialization is highly sensitive to the step-size chosen. With other stepsizes (not shown in figure), it reduces its values too quickly, and fails a significant percentage of the time. The best stepsize is chosen here to show near-optimal performance is possible in the domain.}
  \includegraphics[width=0.6\textwidth]{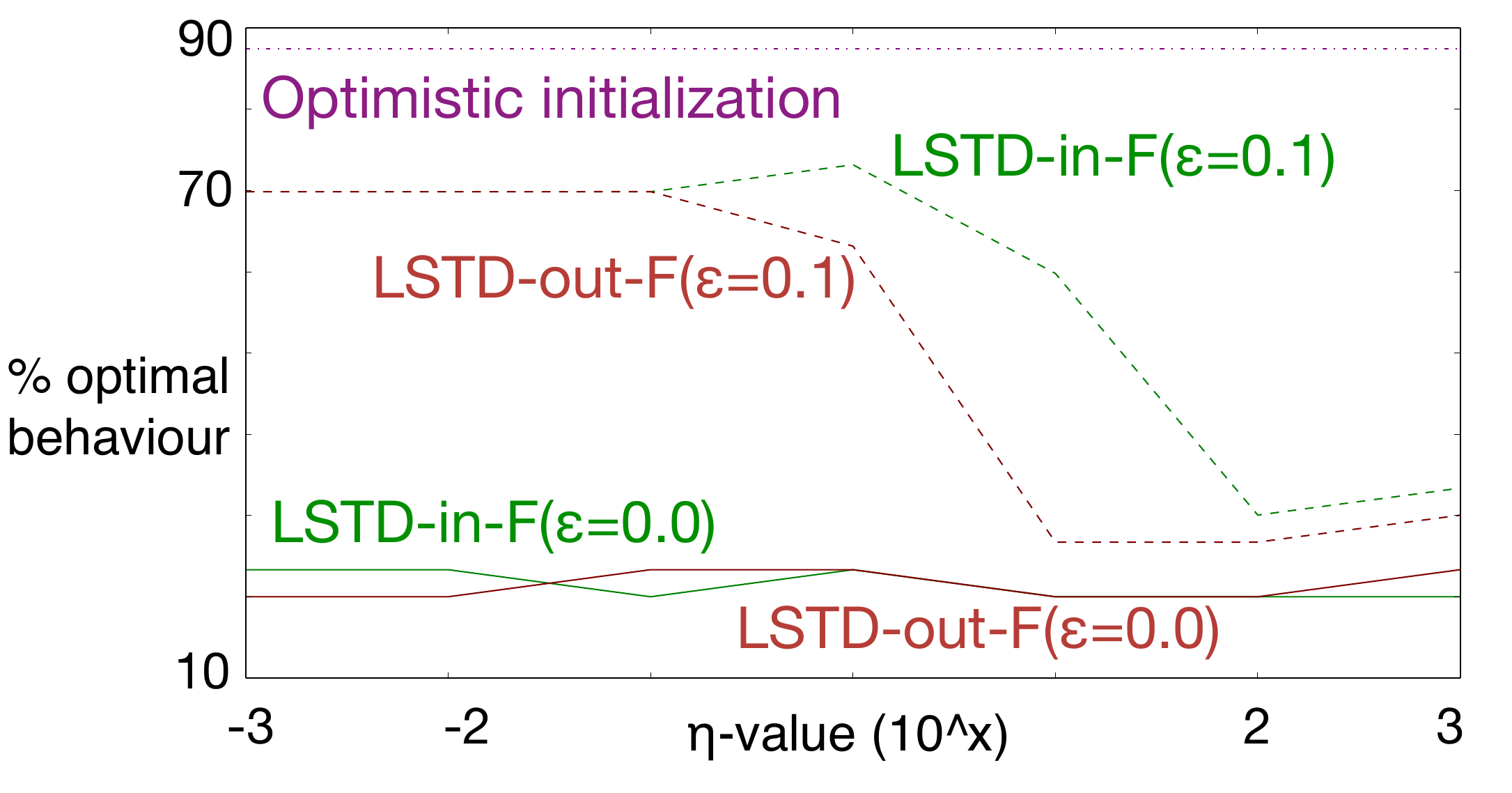}
  \label{fig:1State-CG-etaS}
\end{SCfigure}
%\end{minipage}
%\end{figure}

\section{Optimistic Values Theorem}\label{app_optimisticvalues}

The use of upper confidence bounds on value estimates for exploration has been well-studied and motivated theoretically in online learning \citep{chu2011contextual}. For reinforcement learning, though, there are only specialized proofs for particular algorithms using optimistic estimates \citep{grande2014sample,osband2016generalization}. %To better motivate and appreciate the use of upper confidence bounds for reinforcement learning, we extract the key argument from \citet{osband2016generalization}, which uses the idea of stochastic optimism.
To better appreciate the use of upper confidence bounds for reinforcement learning, we highlight a simple theorem that motivates its utility.
%The result is straightforward, but better highlights how upper confidence-based approaches can be used in reinforcement learning, generally.

Under function approximation, it may not be possible to obtain the optimal policy exactly. Instead, our criterion is to obtain the optimal policy according to the following formulation, assuming greedy-action selection from action-values.
Let $Q^*: \States \times \Actions \rightarrow \RR$ be the action-values for the optimal policy, under the chosen density $d: \States \times \Actions \rightarrow [0, \infty)$ over states and actions
\begin{equation}
Q^* = \argmax_{Q \in \mathcal{Q}} \int_{\States \times \Actions} d(s,a) Q(s,a) ds da
\end{equation}
This optimization does not preclude $d$ being related to the trajectory of optimal policy, but generically allows specification of any density, such as one putting all weight on a set of start states or such as one that is uniform across states and actions to ensure optimality from any point in the space. The optimal policy in this setting is the policy that corresponds to acting greedily w.r.t. $Q^*$; depending on the function space $\mathcal{Q}$, this may only be an approximately optimal policy.
%
%%
%\begin{equation}
%\max_{\weights \in \mathcal{\wspace}} \int_{\States} d_\weights(s) \max_{a \in \Actions} Q_\weights(s,a) ds
%\end{equation}
%%
%%
%where $d_\weights(s)$ is the stationary distribution induced by the weights $\weights$, which determines the policy according because it specifies the action-values with greedy action selection. We omit the $\pi$ as a superscript to $Q_\weights$, because the action-values themselves determine the policy $\pi$, rather than the weights $\weights$ provided approximate action-values for a policy $\pi$.
The design of the agent is directed towards this goal, though we do not explicitly optimize this objective.

Let $\tilde{Q}_t = \hat{Q}_t + \hat{U}_t$ be the estimated action-values plus the confidence interval radius $\hat{U}_t$ on time step $t$, to get the estimated upper confidence bound which the agent uses to select actions. Let $\pi_t$ be the policy induced by greedy action selection on $\tilde{Q}_t$.
\begin{assumption}[Expected Optimism]
At some point $T> 0$,
the action-values at every step $t \ge T$ are optimistic in expectation: $\E[\tilde{Q}_t(S, A)] \ge \E[Q^*(S,A)]$, with expectation according to a specified density $d: \States \times \Actions \rightarrow [0, \infty)$.
\end{assumption}
\begin{assumption}[Shrinking Confidence Interval Radius]
The confidence interval radius $\hat{U}_t$ goes to zero: $\E[\hat{U}_t(S, A)] \le f(t)$ for some non-negative function $f$ with $f(t) \rightarrow 0$.
\end{assumption}
\begin{assumption}[Convergent Action Values]
The estimated action-values $\hat{Q}_t$ approach the true action-values for policy $\pi_t$:
$\left|\E[\hat{Q}_t(S,A) - Q^{\pi_t}(S,A)] \right| \le g(t)$ for some non-negative function $g$ with $g(t) \rightarrow 0$.
\end{assumption}
%

% MARTHAC: Raksha, please make mathcal Q more clear
%Raksha: added to address density used in the assumptions
These assumptions are heavily dependent on the distribution utilized to evaluate the expectation. If the expectations are w.r.t. the stationary distribution induced by the optimal policy ($d^*$), it is easy to see that they could be satisfied - as the density is non-zero only for the optimal state-action pairs.
%Another density that could easily satisfy these assumptions is the start-state action distribution induced by the on-policy control algorithm.
In contrast, if the density is a uniform density over the space, then these assumptions may not be satisfied.
% though this might also make the  possibly rendering the control goal unrealistic in a function approximated framework.%is that true?

Given the three key assumptions, the theorem below is straightforward to prove. However, these three conditions are fundamental, and do not imply each other. Therefore, this result highlights what would need to be shown, to obtain the Optimistic Values Theorem.
For example, Assumption 1 and 2 do not imply Assumption 3, because the confidence interval radius could decrease to zero,
and $\hat{Q}_t$ can still be optimistic in expectation and an over-estimate of values that correspond to a suboptimal policy.
Assumption 1 and 3 do not imply Assumption 2, because $\hat{Q}_t$ could converge to the policy corresponding to acting greedily
w.r.t. $\tilde{Q}_t$, but $\hat{U}_t$ may never fade away. Then, $\tilde{Q}_t$ could still be optimistic in expectation, but the policy $\pi_t$
could be suboptimal because it is acting greedily according to inaccurate, inflated estimates of value $\tilde{Q}_t$.

\begin{theorem}[Optimistic Values Theorem]\label{thm_opt}
Under Assumptions 1, 2 and 3,
\begin{align*}
\E[Q^*(S,A)] - \E[Q^{\pi_t}(S, A)] &\le f(t) + g(t)\\
\text{Regret}(\nsamples) &\defeq \sum_{t=1}^\nsamples \E[Q^*(S,A)] - \E[Q^{\pi_t}(S, A)] \\
&\le \sum_{t=1}^\nsamples f(t) + g(t)
\end{align*}
\end{theorem}
\begin{proof}
Consider the regret across states and actions
\begin{align*}
%Q^*(s,a) - Q^{\pi_t}(s,a) &= Q^*(s,a) - \tilde{Q}_t(s,a) + \tilde{Q}_t(s,a) - Q^{\pi_t}(s,a)\\
\E[Q^*(S,A) - Q^{\pi_t}(S,A)] &= \E[Q^*(S,A) - \tilde{Q}_t(S,A)] + \E[\tilde{Q}_t(S,A) - Q^{\pi_t}(S,A)]\\
&\le \E[\tilde{Q}_t(S,A) - Q^{\pi_t}(S,A)]
\end{align*}
because $\E[Q^*(S,A) - \tilde{Q}_t(S,A)] \le 0$ by Assumption 1.
By Assumptions 2 and 3,
\begin{align*}
\E[\tilde{Q}_t(S,A) - Q^{\pi_t}(S,A)]
&= \E[\hat{Q}_t(S,A) - Q^{\pi_t}(S,A)] +  \E[\hat{U}_t(S,A)] \\
&\le g(t) +  f(t)
\end{align*}
completing the proof.
\end{proof}

This result is intentionally abstract, where the three assumptions could be satisfied in a variety of ways.
%The goal is rather to highlight a general framework for using upper confidence bounds.
% These assumptions have been verified for one algorithm, called RLSVI, under a tabular setting using a finite-horizon specification \citep{osband2016generalization}, which simplifies ensuring stochastic optimism (Assumption 1).
The first assumption would need propagation of optimism as done by many methods which use the principle of optimism in the face of uncertainty in tabular RL.
We hypothesize that the last two assumptions could be addressed with a two-timescale analysis,
with confidence interval radius $\hat{U}_t$ updating more slowly than $\hat{Q}_t$. This would reflect an iterative approach, where the optimistic values are essentially held fixed---such as is done in Delayed Q-learning \citep{grande2014sample}---and $Q^{\pi_t}$ estimated, before then adjusting the optimistic values. The updates to $\hat{Q}_t$, then, would be updated on a faster timescale, converging to $Q^{\pi_t}$, and the upper confidence radius $\hat{U}_t$ updating on a slower timescale.

\begin{algorithm}[h]
\caption{GetOptimisticAction($\xvec_{s,\cdot}$)}\label{alg_UCB}
\begin{algorithmic}
\State $u_{a}  \gets  \sqrt{\left(1 + \tfrac{1}{p}\right)\xvec_{s,a}^\top \Cmat \xvec_{s,a}}$ \hspace{0.2cm} $\forall a \in \Actions$
\State $a = \argmax_{a \in \Actions} \xvec_{s,a}^\top\weights + u_{a}$
\\\Return $a$
\end{algorithmic}
\end{algorithm}

\begin{algorithm}[h]
\caption{UCLS($\lambda$)}\label{alg_UCLS}
\begin{algorithmic}
\State $ \Amat \gets \zerovec$, $\bvec \gets \zerovec$, $\zvec \gets \zerovec$, $\weights \gets \zerovec$
\State $\Bmat \gets \eye$, $\Cmat \gets \eye$,  $\nubarvec \gets \zerovec$, $\counts \gets \onevec$
\State  $\deltaprob = 0.1$, $\eta=10^{-4}$, $\beta = 0.001$, $\cmax = 1.0$
\State $\xvec_{s,\cdot} \gets$ initial state-action features, for any action
\State $a \gets$ GetOptimisticAction$(\xvec_{s,\cdot})$
\Repeat
\State Take action $a$ and observe $\xvec_{s', \cdot}$ and $r$, and $\gamma$
%\State $u_{a'} =   \sqrt{\left(1 + \tfrac{1}{p}\right)\xvec_{s',a'}^\top \Cmat \xvec_{s',a'}}$
%\State $a' \gets$ greedy action according to value estimates given by $\xvec_{s',a'}^\top\weights + u_{a'}$
\State $a' \gets$ GetOptimisticAction$(\xvec_{s',\cdot})$
%\State $\triangleright$ Update weights $\Amat$ and $\bvec$
\State $\delta \gets r + (\gamma \xvec_{s',a'} - \xvec_{s,a})^\top \wvec$
\State $\zvec \gets \gamma\lambda\zvec + \xvec_{s,a}$
\State $\bvec \gets (1-\beta) \bvec + \beta r\zvec$ %\hfill $\triangleright$ Linear-time approximation: step not necessary
\State $\Amat\gets (1-\beta)\Amat + \beta \zvec (\xvec_{s,a} - \gamma \xvec_{s',a'})^\top$
\State $\triangleright$ Update $\Bmat \approx \Amat^\invt$
\State $\alpha = \min\left\{1.0,\frac{0.01}{||\Amat||_F^2||\xvec_{s,a}||_2^2+1.0}\right\}$
\State $\Bmat \gets \Bmat - \alpha\Amat(\Amat^\top\Bmat\xvec_{s,a}-\xvec_{s,a})\xvec_{s,a}^\top$
\State $\triangleright$ Update $\Cmat$
\State $\nubarvec \gets (1-\beta) \nubarvec + \beta \delta \zvec$
\State $\avec \gets \Bmat^\top \nubarvec$
\State $\text{temp} = \cmax$
\State $\cmax = \max(\cmax, \avec_1^2, \ldots, \avec_\xdim^2)$
\If{$\text{temp} \neq \cmax$} \hfill $\triangleright$  Adjust initialization
\State $\Cmat_{ii} \gets \Cmat_{ii} + \counts_i (\cmax - \text{temp})$, $\forall i$
\EndIf
\For{$i$ such that $\zvec_i \neq 0$}
	\State $\counts_i = \counts_i (1-\beta)$
	\For{$j$ such that $\zvec_j \neq 0$}
		\State $\Cmat_{ij} \gets (1-\beta) \Cmat_{ij} + \beta \avec_i\avec_j$
	\EndFor
\EndFor
\State $\triangleright$ Update $\wvec$
\State $\wvec \gets\wvec +  (\Bmat^\top + \eta \eye) (\bvec - \Amat \wvec)$ %\hfill $\triangleright$ Linear-time approximation: $\wvec \gets \wvec + \alpha_l \delta \zvec$
\State $\xvec_{s,a} \gets \xvec_{\svec',a'}$ \ \ \text{ and } \ \ $a \gets a'$
\Until{agent done interaction with environment}
\end{algorithmic}
\end{algorithm}
%\end{minipage}

\newcommand{\deltavar}{\delta_\text{var}}
\newcommand{\wvar}{\mathbf{w}_\text{var}}
\newcommand{\wvarinit}{\mathbf{w}_\text{varInit}}
\newcommand{\vinit}{v_\text{init}}
\newcommand{\alphamean}{\alpha}
\newcommand{\alphavar}{\alpha_\text{var}}

\section{Estimating Upper Confidence Bounds for Policy Evaluation using linear TD} \label{sec_UCBound-linear}
Recall that the TD update \citep{sutton1988learning} processes one sample at a time as $\wvec_{t+1} = \wvec_t + \stepsize \delta_t \zvec_t$ to estimate the solution to the least-squares system $\wvec_\nsamples = \Amat_\nsamples^{-1} \bvec_\nsamples $ in an incremental manner. This is feasible as the following holds:
\begin{align*}
  \wvec_\nsamples &= \Amat_\nsamples^{-1} \bvec_\nsamples\\
  \Amat_\nsamples \wvec_\nsamples &= \bvec_\nsamples\\
  \bigg[\frac{1}{\nsamples} \sum_{t=0}^{\nsamples-1} \zvec_t (\xvec_t - \gamma_{t+1} \xvec_{t+1})^\top \bigg] \wvec_\nsamples &= \bigg[\frac{1}{\nsamples} \sum_{t=0}^{\nsamples-1} \zvec_t r_t \bigg] \\
  \sum_{t=0}^{\nsamples-1} \zvec_t (r_t + \gamma_{t+1} \xvec_{t+1}^\top \wvec_\nsamples - \xvec_{t}^\top \wvec_\nsamples) &= 0 \\
  \sum_{t=0}^{\nsamples-1} \zvec_t \delta_t &= 0
\end{align*}
Therefore, $\wvec_t$ is updated incrementally with a constant step-size towards minimizing this error stochastically.

Given this incremental method to estimate a least-squares solution, we can notice that the covariance matrix is the outer-product of the solution to a similar least-squares system, $\Amat_\nsamples^{-1} \nubarvec_\nsamples$. The solution to this least-squares system is denoted by $\wvar$, and can be estimated incrementally as:
\begin{align*}
  {\wvar}_{t+1} = {\wvar}_t + \alpha {\deltavar}_t \zvec_t
\end{align*}
where, ${\deltavar}_t = \delta_t + \gamma_{t+1} \xvec_{t+1}^\top {\wvar}_t - \xvec_{t}^\top {\wvar}_t$.

Therefore, for a given policy, the true action-values satisfy the following:
\begin{align*}
\xvec^\top \weights^* &\le \xvec^\top \weights_\nsamples + \sqrt{\tfrac{\deltaprob+1}{\deltaprob}}\sqrt{\xvec^\top {\wvar}_{\nsamples} {\wvar}_{\nsamples}^{\top}\xvec} \label{UCLS-upperbound}
% &\le  \xvec^\top \weights_\nsamples + \max\left\{\tfrac{1}{\sqrt{\deltaprob}}, 2\right\}\sqrt{\xvec^\top \E[\Amat_\nsamples^\inv\nubarvec_\nsamples \nubarvec_\nsamples^\top \Amat_\nsamples^{-\top}]  \xvec } \label{UCLS-upperbound}
\end{align*}

Similarly a linear variant of GV-UCB can be obtained as the upper bound again consists of an outer-product to a different least-squares system $\Amat_\nsamples^\inv\zbarvec_\nsamples$. But as shown in Figures \ref{fig:RS-pSen} and \ref{fig:LC-APP}, GV-UCB, the quadratic version, can be highly sample inefficient, which may worsen with the linear variant, GV-UCB-L. Therefore, we do not provide an algorithm, or any empirical results for GV-UCB-L here.

\section{UCLS-L: Estimating upper confidence bounds for linear TD in control}\label{sec_UCLS-L}
In the same spirit as UCLS utilizes the policy evaluation upper-bound of LSTD for control, with a slowly changing control policy, UCLS-L utilizes the policy evaluation upper-bound of linear TD for control. At each step, UCLS-L, given in Algorithm \ref{alg_UCLS_linear}, uses a stochastic update to estimate mean action-values, and their corresponding contextual-variance estimates. These stochastic updates use fixed, and if necessary are different, step-sizes ($\alpha$, and $\alphavar$ respectively), instead of a closed-form solution as done by UCLS. The rate of change of the policy in UCLS-L is controlled by the step-size, unlike in UCLS which utilizes weighted forms of experience samples in $\Amat$ and $\bvec$. Therefore, UCLS-L can be sensitive to the step-sizes, but adapt more quickly to a changing feature-space. Further, in order to account for underestimates of variances, UCLS-L uses another vector $\wvarinit$, in a similar spirit as UCLS's retroactive initialization of covariance estimates. Additionally, as these upper-bounds are estimated incrementally, they can be quite loose, specifically so in the linear framework. Therefore, instead of choosing the best parameter $p$, we can choose a parameter $\bar{p} = \sqrt{1 + \frac{1}{p}}$: the loss of theoritical interpretation of the upper-bound is traded-off for better empirical performance.

\begin{algorithm}[h]
\caption{GetOptimisticActionLinear($\xvec_{s,\cdot}$)}\label{alg_UCB_linear}
\begin{algorithmic}
% \State $u_{a}  \gets  \sqrt{\left(1 + \tfrac{1}{p}\right) (\xvec_{s,a}^\top (\wvar+\wvarinit))^2}$ \hspace{0.2cm} $\forall a \in \Actions$
\State $u_{a}  \gets  \sqrt{\left(1 + \tfrac{1}{p}\right) \big((\xvec_{s,a}^\top\wvar)^2 + ||\xvec_{s,a}||^2_{\eye\wvarinit}\big)}$ \hspace{0.2cm} $\forall a \in \Actions$
\State $a = \argmax_{a \in \Actions} \xvec_{s,a}^\top\weights + u_{a}$
\\\Return $a$
\end{algorithmic}
\end{algorithm}

\begin{algorithm}[h]
\caption{UCLS-L($\lambda$)}\label{alg_UCLS_linear}
\begin{algorithmic}
\State $\deltaprob = 0.1$, $\beta = 0.001$, $\vinit = 1.0$, $\alphamean = 0.01$, $\alphavar = 0.1$
\State $ \weights \gets \zerovec$, $\wvar \gets \zerovec$, $\wvarinit \gets \onevec*\vinit$, $\counts \gets \onevec$
\State $\xvec_{s,\cdot} \gets$ initial state-action features, for any action
\State $a \gets$ GetOptimisticActionLinear$(\xvec_{s,\cdot})$
\Repeat
\State Take action $a$ and observe $\xvec_{s', \cdot}$ and $r$, and $\gamma$
%\State $u_{a'} =   \sqrt{\left(1 + \tfrac{1}{p}\right)\xvec_{s',a'}^\top \Cmat \xvec_{s',a'}}$
%\State $a' \gets$ greedy action according to value estimates given by $\xvec_{s',a'}^\top\weights + u_{a'}$
\State $a' \gets$ GetOptimisticActionLinear$(\xvec_{s',\cdot})$
%\State $\triangleright$ Update weights $\Amat$ and $\bvec$
\State $\delta \gets r + (\gamma \xvec_{s',a'} - \xvec_{s,a})^\top \wvec$
\State $\deltavar \gets \delta + (\gamma \xvec_{s',a'} - \xvec_{s,a})^\top \wvar$
\State $\zvec \gets \gamma\lambda\zvec + \xvec_{s,a}$
\State $\triangleright$ Update $\wvar$ and $\wvarinit$
\State $\wvar \gets\wvar +  \alphavar \deltavar \zvec$
\State $\text{temp} = \vinit$
\State $\vinit = \max(\vinit, {\wvar}_{1}^2, \ldots, {\wvar}_\xdim^2)$
\If{$\text{temp} \neq \vinit$} \hfill $\triangleright$  Adjust initialization
\State ${\wvarinit}_{i} \gets {\wvarinit}_{i} + \counts_i (\vinit - \text{temp})$, $\forall i$
\EndIf
\For{$i$ such that $\zvec_i \neq 0$}
	\State $\counts_i = \counts_i (1-\beta)$
  \State ${\wvarinit}_{i} \gets (1-\beta)*{\wvarinit}_{i}$, $\forall i$
\EndFor
\State $\triangleright$ Update $\wvec$
\State $\wvec \gets\wvec +  \alphamean \delta \zvec$
\State $\xvec_{s,a} \gets \xvec_{\svec',a'}$ \ \ \text{ and } \ \ $a \gets a'$
\Until{agent done interaction with environment}
\end{algorithmic}
\end{algorithm}

With this, we investigate UCLS-L as a substitue to UCLS in the three benchmark domains. For UCLS-L, both $p$ and $\bar{p}$ is swept, from which the best parameter is selected scale the uncertainity unstemiate, along with the learning rates $\alphamean$ and $\alphavar$. The experiment configuration and the domains are the same as used in UCLS. The results are presented in Figure $\ref{fig:LC-L}$. UCLS-L does reasonably well in all the domains. While it experiences more regret in Puddle World, and River Swim during early learning, by the end of the steps budget, it learns the optimal policy. In Sparse Mountain Car, surprisingly, UCLS-L learns much faster and a better policy than UCLS. This can be attributed to the fact that the parameter $p$ in UCLS was not swept, whereas in UCLS-L we did sweep to find the best parameter to scale the variance estimate. As the domain is a sparse-reward domain, the variance estimates play a significant role in influencing exploratory behaviour, and therefore optimizing for $p$ would improve UCLS' performance. Nonetheless, these results show UCLS-L to be a promising algorithm for linear complexity based control, and warrant further evaluation of it.

\begin{figure*}[ht]
\centering
  \includegraphics[width=0.9\textwidth]{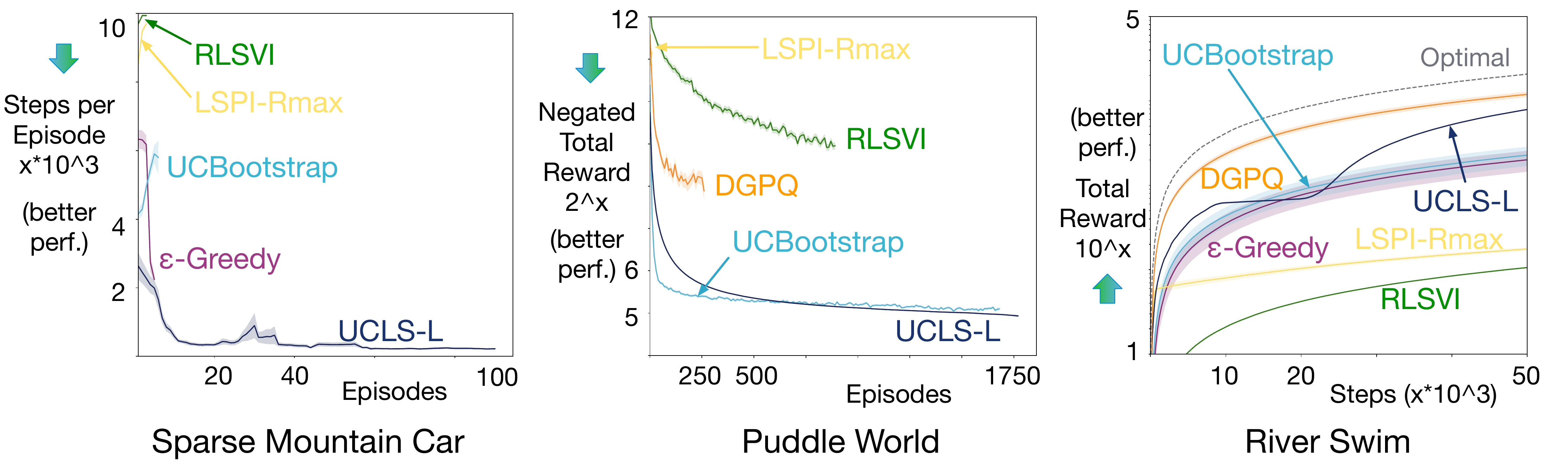}
\caption{A comparison of speed of learning in Sparse Mountain Car, Puddle World and River Swim. In plots (a) and (b) lower on y-axis are better, whereas in (c) curves higher along y-axis are better. Sparse Mountain Car and Puddle World are episodic problems with a fixed experience budget. Thus the length of the lines in plots (a) and (b) indicate how many episodes each algorithm completed over 50,000 steps, and the height on the y-axis indicates the quality of the learned policy---lower indicates better performance. Note RLSVI did not show significant learning after 50,000 steps. The RLSVI result in Puddle World uses a budget of 1 million.}
\label{fig:LC-L}
\end{figure*}

\newcommand{\rbar}{\bar{{r}}}
\newcommand{\rsqbar}{\bar{{r^2}}}
\newcommand{\dbarvec}{\bar{{\mathbf{d}}}}
\newcommand{\drbarvec}{\bar{{\mathbf{d_r}}}}
\newcommand{\dmat}{\bar{{\mathbf{D}}}}
\newcommand{\nubar}{\bar{{\nu}}}
\newcommand{\nusqbar}{\bar{{\nu^2}}}

\begin{algorithm}
\caption{GetOptimisticActionGlobal($\xvec_{s,\cdot}$)}\label{alg_UCBG}
\begin{algorithmic}
\State $u_{a}  \gets  \sigma \sqrt{\left(1 + \tfrac{1}{p}\right)\xvec_{s,a}^\top \Cmat \xvec_{s,a}}$ \hspace{0.2cm} $\forall a \in \Actions$
\State $a = \argmax_{a \in \Actions} \xvec_{s,a}^\top\weights + u_{a}$
\\\Return $a$
\end{algorithmic}
\end{algorithm}

\begin{algorithm}[t]
\caption{GV-UCB($\lambda$)}\label{alg_UCLS-G}
\begin{algorithmic}
\State $ \Amat \gets \zerovec$, $\bvec \gets \zerovec$, $\zvec \gets \zerovec$, $\weights \gets \zerovec$,
\State $\Bmat \gets \eye$, $\Cmat \gets \eye$,  $\zbarvec \gets \zerovec$
\State  $p = 0.01$,$\eta=10^{-4}$, $\beta = 0.001$
\State  $\sigma = 1.0$, $\rbar = 0.0$, $\rsqbar = 100.0$, $\dbarvec \gets \zerovec$, $\drbarvec \gets \zerovec$, $\dmat \gets \zerovec$
\State $\xvec_{s,\cdot} \gets$ initial state-action features, for any action
\State $a \gets$ GetOptimisticActionGlobal$(\xvec_{s,\cdot})$
\Repeat
\State Take action $a$ and observe $\xvec_{s', \cdot}$ and $r$, and $\gamma$
\State $a' \gets$ GetOptimisticActionGlobal$(\xvec_{s',\cdot})$
\State $\delta \gets r + (\gamma \xvec_{s',a'} - \xvec_{s,a})^\top \wvec$
\State $\zvec \gets \gamma\lambda\zvec + \xvec_{s,a}$
\State $\bvec \gets (1-\beta) \bvec + \beta r\zvec$
\State $\Amat\gets (1-\beta)\Amat + \beta \zvec (\xvec_{s,a} - \gamma \xvec_{s',a'})^\top$
\State $\triangleright$ Update $\Cmat$
\State $\zbarvec \gets (1-\beta) \zbarvec + \beta \zvec$
\State $\avec \gets \Bmat^\top \zbarvec$
\For{$i$ such that $\zvec_i \neq 0$}
	\For{$j$ such that $\zvec_j \neq 0$}
		\State $\Cmat_{ij} \gets (1-\beta) \Cmat_{ij} + \beta \avec_i\avec_j$
	\EndFor
\EndFor
\State $\triangleright$ Update $\sigma$
\State $\rbar \gets (1 - \beta) \rbar + \beta r$
\State $\rsqbar \gets (1 - \beta) \rsqbar + \beta r^2$
\State $\dbarvec \gets (1 - \beta) \dbarvec + \beta (\xvec_{s,a} - \gamma \xvec_{s',a'})$
\State $\drbarvec \gets (1 - \beta) \drbarvec + \beta r(\xvec_{s,a} - \gamma \xvec_{s',a'})$
\State $\dmat \gets (1 - \beta) \dmat + \beta (\xvec_{s,a} - \gamma \xvec_{s',a'})(\xvec_{s,a} - \gamma \xvec_{s',a'})^\top$
\State $\nubar = \rbar - \dbarvec^T \wvec$
\State $\nusqbar = \rsqbar - 2 \drbarvec^T \wvec + \wvec^\top \dmat \wvec$
\State $\sigma = \sqrt{\nusqbar - \nubar^2}$
\State $\triangleright$ Update $\wvec$ and $\Bmat \approx \Amat^\invt$
\State $\alpha = \min\left\{1.0,\frac{0.01}{||\Amat||_F^2||\xvec_{s,a}||_2^2+1.0}\right\}$
\State $\Bmat \gets \Bmat - \alpha\Amat(\Amat^\top\Bmat\xvec_{s,a}-\xvec_{s,a})\xvec_{s.a}^\top$
\State $\wvec \gets\wvec +  (\Bmat + \eta \eye) (\bvec - \Amat \wvec)$
\State $\xvec_{s,a} \gets \xvec_{\svec',a'}$ \ \ \text{ and } \ \ $a \gets a'$
% TODO: Update to be the actual algorithm that we use
%\State $\dvec \gets (\xvec - \gamma \xvec')^\top$
%\State $\Amatinv \gets \left[\Amatinv - \frac{(\Amatinv \zvec)(\dvec^\top \Amatinv)}{1+(\dvec^\top \Amatinv)\zvec}\right]$
%\State $\xvec \gets \xvec'$
%\State $\weights = \Amatinv \bvec$
\Until{agent done interaction with environment}
\end{algorithmic}
\end{algorithm}

\section{Details about other algorithms}

\subsection{Global variance UCB}\label{sec_gvucb}

Based on Corollary \ref{cor_global} to estimate a  global variance $\sigma^2$,
it is possible that the noise may not be 0-mean during the learning process. We account for this by estimating mean of $\nu_t$ as well. We know $\nu_t \sim \mathcal{N}(\bar{\nu_t}, \sigma_t^2)$. Therefore:
\begin{align*}
\bar{\nu}_{t+1} &= E[r_{t+1}] &&- E[\vec{x}_t - \gamma \vec{x}_{t+1}]^\top \weights_t \\
\bar{\nu^2}_{t+1} &= E[r_{t+1}^2] &&- 2E[r_{t+1}(\vec{x}_t - \gamma \vec{x}_{t+1})]^\top \weights_t\\
& &&+ \weights_t^\top E[(\vec{x}_t - \gamma \vec{x}_{t+1}) (\vec{x}_t - \gamma \vec{x}_{t+1})^\top] \weights_t
\end{align*}
These expected values are maintained incrementally. Utilizing this, $\sigma_{t+1}^2 = \bar{\nu^2}_{t+1} - \bar{\nu}_{t+1}^2$. We refer to Global variance UCB as GV-UCB. The algorithm is given in Algorithm \ref{alg_UCLS-G}.

\subsection{Bootstrapped upper confidence bounds}

%We can use the approach in \citep{white2010interval}.
%The main goal is to improve the approach for sparse representations, which will likely require exploring bootstrapped confidence intervals for multivariate random variables. Previously, we maintained a confidence interval separately around each parameter, and combined these univariate intervals after the fact; there is likely a more principled approach, that maintains the dependent structure between features and captures the fact that parts of the feature vector may not change for long periods of time. An addition of similarity to re-weight saved windows to make it more likely to bootstrap those may be required.  Another option is to summarize or average previous parameters to make them more dense, and so the bootstrapped samples will not all be the same.

The strategy for action selection which utilizes bootstrapped confidence intervals, as proposed by \citet{white2010interval}, is given in Algorithm \ref{alg_bootstrapglobal}. This action selection strategy can be used in conjunction with any learning algorithm to guide on-policy control. The algorithm requires a window of recent $\wvec$'s. The window can be maintained with a circular queue. The window is updated after each learning step of the main algorithm, resulting in a new $\wvec_t$ in the queue.
The original UCBootstrap paper proposed both a global and a sparse updating mechanism, where only the global approach was theoretically justified.
The sparse mechanism was used to reduce the number of parameters stored, particularly by taking advantage of tile-coding representations.
We found in our experiments that the global approach worked just as well as the sparse approach, and so we include only the simpler,
theoretically justified algorithm.

\begin{algorithm}
\caption{UCBootstrap($\xvec_{s,\cdot}$) select action from state features $\xvec_{s,\cdot}$ at time $t$}\label{alg_bootstrapglobal}
$l = $ block length, $B$ = number of bootstrap resamples, $\window$ = number (window) of value functions weights to store and confidence level $\alpha$\\
examples: $l = 10$, $B = 50$, $\window = 100$, $\alpha = 0.05$
\begin{algorithmic}
\State $M \gets  \lfloor w/l \rfloor$ \Comment{num of length $l$ blocks to sample with replacement and concatenate}
\For{each action $a$}
\State $Q_N \gets \{\wvec_{t-\window}^\top \xvec_{s,a},\ldots, \wvec_{t-1}^\top \xvec_{s,a}\}$
%\Comment{$\{Q(\wvec_{t-\window},\xvec_{s,a}),\ldots Q(\wvec_{t-1},\xvec_{s,a})\} $}
\State $\bar{Q}_N \gets $ mean($Q_N$)  \Comment{The mean value for this $(s,a)$, given the window of recent weights}
\State Blocks = {\small$\Big\{\{[Q_N[0],\ldots,Q_N[l\text{-}1]\}, \{[Q_N[1],\ldots,Q_N[l]\}, $ }
\State \hspace{3.5cm}{\small $\ldots, [Q_N[w\text{-}l], \ldots, Q_N[w\text{-}1]]\Big\}$ }
%\Comment{Equivalent to blocks of $A$, represented as ${A^*_{1},A^*_{2}..A^*_{M}}$}
\ForAll{$i = 1$ to $B$}
\ForAll{$j = 1$ to $M$}
\State $A^*_{j} \gets$ random block from Blocks (chosen with replacement)
\EndFor
\State $A \gets (A^*_{1}, A^*_{2},\ldots,A^*_{M})$  \Comment{Concatenate blocks}
\State $ T^*_i = \frac{1}{lM} \sum_{k=1}^{lM} A[k]$ \Comment{$i$th bootstrap estimate is the mean of the $M$ concatenated blocks}
\EndFor
\State $T \gets$ sort($\{T^*_1,\ldots,  T^*_B\}$) \Comment{ascending order}
\State $j \gets \lfloor \frac{B\alpha}{2} + \frac{\alpha+2}{6}\rfloor$  \Comment{$j$ is the position of the critical samples to help estimate the continuous sample quantile}
%\Comment{indexing of $T$ starts at 0, so can have $j = 0$}
\State $r \gets \frac{B\alpha}{2} + \frac{\alpha+2}{6} - j$  \Comment{$r$ is the remainder}
\State $T^*_{\alpha/2} \gets (1-r)T^*_j + rT^*_{j+1}$ \Comment{the $\alpha/2$ sample quantile}
\State $u_a \gets 2\bar{Q}_N - T^*_{\alpha/2}$
\EndFor
\State $a = \argmax_{a \in \Actions} u_{a}$
\\\Return $a$
\end{algorithmic}
\end{algorithm}

\begin{algorithm}
\caption{DGPQ($k(\cdot,\cdot), d(\cdot,\cdot), L_{Q}, Env, \Actions, R_{max},$ $s_0, \gamma, \sigma^2, \sigma^2_{tol}, \epsilon$)}\label{alg_DGPQ}
$k(\cdot,\cdot), d(\cdot,\cdot)$ are typically the RBF w/ bandwidth = $\sigma^2$ and euclidean distance respectively.
\\ $L_{Q}$ correlates with exploration.
\\ $\Actions$ is the set of possible actions.
\\ $\gamma$ is the discount factor.
\\ $\sigma^2_{tol}$ is the tolerance of induced variance of using a new point to update a GP
\\ Found useful ranges for parameters during sweeps:
\\ \hspace{1cm} $\sigma^2 \in [0.001,0.5]$,$\sigma^2_{tol} \in [0.01,0.1]$, $\epsilon \in [0.01,0.1]$, $L_{Q}\in[1,20]$
\begin{algorithmic}[1]
\State $\hat{Q}(s,a) \defeq \min$(\par
\hskip \algorithmicindent$V_{max}$,\par
\hskip \algorithmicindent$\min\limits_{(s_i,a) \in \hat{Q}_{a}.BV} \left\{ [\hat{\mu}_i + L_Q d((s,a),(s_i,a))] \right\}$\par
\hskip \algorithmicindent )
\For{ $a \in A$ }
\State $\hat{Q}_a.BV = \emptyset$
\State $GP_a = GP.init(\mu = \frac{Rmax}{1-\gamma}, k(\cdot,\cdot))$
\EndFor
\For   { $t \in [0,T]$ }
\State $a_t = \arg\!\max\limits_a \hat{Q}(s,a)$
\State //take action $a_t$ in state $s_t$, observe $(s_{t+1},r_{t})$
\State $(s_{t+1},r_{t}) = Env(s_t,a_t)$
\State $q_t = r_t + \gamma \max\limits_a \hat{Q}(s_{t+1})$
\State $\sigma_1^2 = GP_{a_t}.variance(s_t)$
\State //If the new sample is not well covered by $GP_{a_t}$
\If    {$\sigma_1^2 > \sigma_{tol}^2$}
\State $GP_{a_t}.update(s_t,q_t)$
\EndIf
\State $\sigma_2^2 = GP_{a_t}.variance(s_t)$
\State //If the $GP_{a_t}$ now well covers a previously unknown state and the new approximation is $2\epsilon$ less than what is found in $\hat{Q}$ (i.e. is a less optimistic estimate).
\If    {$\left\{\sigma_1^2 > \sigma_{tol}^2 \ge \sigma_2^2\right\}$  \textbf{and} \par
        $\left\{\hat{Q}_{a_t}(s_t) - GP_{a_t}.mean(s_t) > 2\epsilon\right\}$}
\State $\mu = GP_{a_t}.mean(s_t) + \epsilon$
\State $\hat{Q}_{a_t}.BV.add((s_t,a_t),\mu)$
\For {$((s_j,a_t),\mu_j) \in \hat{Q}_{a_t}.BV$}
\If {$\mu_j \le \mu + L_{Q} d((s_t,a_t),(s_j,a_t))$}
\State $\hat{Q}_{a_t}.BV.delete(((s_j,a_t),\mu_j))$
\EndIf
\EndFor
\State //To prevent slow learning or halted learning reset the current GPs and initialize to the current estimates.
\State $\forall a \in A, GP_{a} = GP.init(\hat{\mu} = \hat{Q_a}, k(\cdot,\cdot))$
\EndIf
\EndFor
\end{algorithmic}
\end{algorithm}

\newcommand{\featcounts}{\mathbf{f}}

\begin{algorithm}
\caption{IsKnown($s,a$)}\label{alg_known}
\begin{algorithmic}[1]
\State // Uses the minimum count of the features for a state, to decide if $s,a$ is known
\State // If $a$ not given, sums over all $a$
\State $m = 5$
\If{$a$ not given}
\State $\featcounts \gets \sum_{a} \counts(\xvec_{s,a}) \in \RR^\xdim$
\Else
\State $\featcounts \gets \counts(\xvec_{s,a}) \in \RR^\xdim$
\EndIf
\If{$\min(\featcounts) > m$}
\State \Return ``Known''
\Else
\State \Return ``Not Known''
\EndIf
\end{algorithmic}
\end{algorithm}

% \begin{align*}
% \hat{Q}(s,a) &= min \left\{ \min_{(s_i,a)\in BV} [\hat{\mu}_i + L_Q d((s,a),(s_i,a))], V_{max}\right\}\\
% \end{align*}

\newcommand{\gmax}{G_{\text{max}}}

\begin{algorithm}
\caption{Incremental LSPI-Rmax($m$)}\label{alg_lspi}
\begin{algorithmic}[1]
\State $ \Amat \gets \zerovec$, $\bvec \gets \zerovec$, $\zvec \gets \zerovec$, $\weights \gets \zerovec$,
\State $\Bmat \gets \eye$, $\counts \gets \zerovec$
\State  $\eta=10^{-4}$, $\beta = 0.001$, $\lambda = 0$, $\gmax = \rmax/(1-\gamma)$ if continuing or $\gamma \neq 1$, else $\gmax = \rmax h$ for a predicted maximum episode length (e.g., $h = 10000$).
\State $\xvec_{s,\cdot} \gets$ initial state-action features, for any action
\State $a \gets$ greedy action according to value estimates given by $\xvec_{s,a}^\top\weights$
\Repeat
\State Take action $a$ and observe $\xvec_{s', \cdot}$ and $r$, and $\gamma$
\State $a' \gets$ greedy action according to value estimates given by $\xvec_{s',a'}^\top\weights$
\State $\zvec \gets \gamma\lambda\zvec + \xvec_{s,a}$
\If{IsKnown($s,a$)}
	\If{IsKnown($s'$)}
		\State $\Amat\gets (1-\beta)\Amat + \beta \zvec (\xvec_{s,a} - \gamma \xvec_{s',a'})^\top$
		\State $\bvec \gets (1-\beta) \bvec + \beta r\zvec$
	\Else
		\State $\Amat\gets (1-\beta)\Amat + \beta \xvec_{s,a} \xvec_{s,a}^\top$
		\State $\bvec \gets (1-\beta) \bvec + \beta (r + \gamma \gmax) \xvec_{s,a}$
	\EndIf
\Else
	\State $\Amat\gets (1-\beta)\Amat + \beta \xvec_{s,a} \xvec_{s,a}^\top$
	\State $\bvec \gets (1-\beta) \bvec + \beta \gmax \xvec_{s,a}$
\EndIf
\For{$\forall \tilde{a} \in A \backslash a$}
	\If{$!\text{IsKnown}(s,\tilde{a})$}
		\State $\Amat\gets (1-\beta)\Amat + \beta \xvec_{s,\tilde{a}} \xvec_{s,\tilde{a}}^\top$
		\State $\bvec \gets (1-\beta) \bvec + \beta \gmax \xvec_{s,\tilde{a}}$
	\EndIf
\EndFor
\State $\counts \gets \counts + \xvec_{s,a}$
\State $\alpha = \min\left\{1.0,\frac{0.01}{||\Amat||_F^2||\xvec_{s,a}||_2^2+1.0}\right\}$
\State $\Bmat \gets \Bmat - \alpha\Amat(\Amat^\top\Bmat\xvec_{s,a}-\xvec_{s,a})\xvec_{s,a}^\top$
\State $\wvec \gets\wvec +  (\Bmat + \eta \eye) (\bvec - \Amat \wvec)$
\State $\xvec_{s,a} \gets \xvec_{\svec',a'}$ \ \ \text{ and } \ \ $a \gets a'$
\Until{agent done interaction with environment}
\end{algorithmic}
\end{algorithm}

\begin{algorithm}
\caption{LSTD($\lambda$) with Sherman-Morrison update}\label{alg_LSTD_SM}
\begin{algorithmic}[1]
\State $ \Amat^\inv \gets \frac{1}{\eta} \eye$, $\bvec \gets \zerovec$, $\zvec \gets \zerovec$, $\weights \gets \zerovec$,
\State $\xvec_{s,\cdot} \gets$ initial state-action features, for any action
\State $a \gets$ $\epsilon$-greedy action according to value estimates given by $\xvec_{s,a}^\top\weights$
\Repeat
%\State Take $\epsilon$-greedy action according to value estimates given by $\xvec_a^\top\weights_a$, observe $\xvec'$ and $r$.
\State Take action $a$ and observe $\xvec_{s', \cdot}$ and $r$, and $\gamma$
\State $a' \gets$ $\epsilon$-greedy action according to value estimates given by $\xvec_{s',a'}^\top\weights$
\State $\zvec \gets \gamma\lambda\zvec + \xvec_{s,a}$
\State $\beta = \tfrac{1}{t}$
\State $\bvec \gets \bvec + \beta(r\zvec - \bvec)$
\State $\vvec \gets \left((\xvec_{s,a} - \gamma \xvec_{s',a'})^\top \Amat^\inv \right)^\top$
\State $\Amat^\inv \gets  \frac{1}{(1-\beta)}\Amat^\inv - \frac{\frac{\beta}{(1-\beta)}\Amat^\inv \zvec \vvec^\top}{(1-\beta) + \beta \vvec^\top \zvec}$
\State $\weights \gets \Amat^\inv \bvec$
\State $\xvec_{s,a} \gets \xvec_{\svec',a'}$ \ \ \text{ and } \ \ $a \gets a'$
\Until{agent done interaction with environment}
\end{algorithmic}
\end{algorithm}

\begin{algorithm}
\caption{LSTD($\lambda$) with Conjugate Gradient}\label{alg_LSTD_CG}
\begin{algorithmic}
\State $ \Amat \gets \zerovec$, $\bvec \gets \zerovec$, $\zvec \gets \zerovec$, $\weights \gets \zerovec$,
\State $\xvec_{s,\cdot} \gets$ initial state-action features, for any action
\State $a \gets \epsilon$-greedy action according to value estimates given by $\xvec_{s,a}^\top\weights$
\Repeat
\State Take action $a$ and observe $\xvec_{s', \cdot}$ and $r$, and $\gamma$
\State $a' \gets \epsilon$-greedy action according to value estimates given by $\xvec_{s',a'}^\top\weights_{a'}$
\State $\zvec \gets \gamma\lambda\zvec + \xvec_{s,a}$
\State $\beta_t = \tfrac{1}{t}$ \hfill $\triangleright$ or constant such as $\beta_t = 0.01$
\State $\bvec \gets \bvec + \beta_t(r\zvec - \bvec)$
\State $\Amat\gets \Amat + \beta_t (\zvec (\xvec_{s,a} - \gamma \xvec_{s',a'})^\top-\Amat)$
\State $\wvec \gets \text{ConjugateGradient}(\Amat + \etasm\eye, \bvec, \wvec, \etat)$
\State $\xvec_{s,a} \gets \xvec_{\svec',a'}$ \ \ \text{ and } \ \ $a \gets a'$
\Until{agent done interaction with environment}
\end{algorithmic}
\end{algorithm}

\newcommand{\rprev}{\textbf{r}}
\newcommand{\rnext}{\textbf{r}'}
\begin{algorithm}
\caption{Conjugate Gradient($\Amat, \bvec, \wvec, \etat$)}\label{alg_CG}
\begin{algorithmic}[1]
\State $\text{tol} = 0.001$
\State $\tilde{\Amat} = \Amat^T\Amat + \etat\eye$
\State $\rprev \gets \bvec - \tilde{\Amat}\wvec$
\State $\dvec \gets \rprev$
\Repeat
\State $\alpha \gets \frac{\rprev^\top\rprev}{\dvec^\top\Amat\dvec}$
\State $\wvec \gets \wvec + \alpha \dvec$
\State $\rnext\gets \rprev - \alpha \tilde{\Amat}\dvec$
\State $\beta \gets \frac{\rnext^\top \rnext}{\rprev^\top\rprev}$
\State $\dvec \gets \rnext + \beta \dvec$
\Until{CG converged ($||\rnext||_2^2 \leq \text{tol}$) or a fixed number of steps reached}
\\\Return $\wvec$
\end{algorithmic}
\end{algorithm}

\subsection{DGPQ}
Another approach to exploration is found in a model-free algorithm using gaussian processes named Delayed-GPQ (DGPQ) \citep{grande2014sample}. The pseudocode for DGPQ is in Algorithm \ref{alg_DGPQ}. Any algorithm can be used to train the Gaussian processes, and for this paper we use the same algorithm as in \citep{grande2014sample}. The initialization of this algorithm requires the maximum reward and value, but for ease of use we transform the reward signal to $r_{new} = r - R_{max}$ so the means of the gaussian processes can be initialized to zero and $V_{max} = 0$.
%Actually I think it is better outside of the algorithm. Makes it a bit clearer.

A major problem with DGPQ is the large number of parameters needed to be set properly. Some intuition on setting these parameters can be found in \citep{grande2014sample} as well as in algorithm \ref{alg_DGPQ}. As some guidance the width of the kernel determines how much a sample can generalize to other states, the thresholds ($\sigma^2_{tol}, \epsilon$) determine how often we swap for new experience in the set basis vectors, and the Lipschitz constant $L_Q$ tunes the tradeoff between exploration and exploitation.

\subsection{LSPI-Rmax}
LSPI-Rmax \cite{li2009online} combines LSPI \cite{lagoudakis2003least} with Rmax \cite{brafman2003rmax} for online control in continuous state-spaces. Exploration is encouraged by determining the \textit{knowness} of a transition, utilizing kernels.
LSPI algorithm is designed for a batch setting, where the LSTD solution is computed in closed form for staged batches of data. However, because it accumulates optimistic values, it can be simply converted into an online algorithm using incremental updates to the matrix $\Amat$ and $\bvec$, as done in \citet{li2009online}.

We summarize this extension in pseudocode as Algorithm \ref{alg_lspi}.
Until states become known, the algorithm estimates action-values that predict the maximum possible return; once a state becomes known, it starts to use actual rewards sampled from the environment.
To estimate the \textit{knowness} of a state under function approximation, we use feature counts. Each state has a set of active features; the active feature with the minimum count reflects an upper bound on the number of times that this state has been seen. Once a states active features have been seen frequently enough, it becomes known.

\subsection{RLSVI}
RLSVI \citep{osband2016generalization} is an algorithm that maintains a distribution over the possible value functions. The value functions are assumed to be linearly parametrized. While the main algorithm proposed uses a finite-horizon assumption, a modified version proposed in the Appendix of the paper does not, and this is the version used in the experiments here.

\iffalse
\begin{figure*}[!htb]
\includegraphics[width=\textwidth ,scale=0.25]{figures/APP-LC.pdf}
\caption{Learning curves on three domains - in Sparse Mountain Car, Puddle World, and River Swim. In the first two plots lower on y-axis indicates better performance, whereas in the right-most plot higher along y-axis is better.}
\label{fig:LC-APP}
\end{figure*}

\begin{figure}
\includegraphics[width=0.9\textwidth ,scale=0.25]{figures/APP-RUNS.pdf}
\caption{Best and worst run curves for DGPQ and UCLS. From top to bottom: Sparse Mountain Car, Puddle World, River Swim}
\label{fig:RUNS}
\end{figure}
\fi

\section{Alternative updates for LSTD}

The update for $\wvec$ using $\Amat$ and $\bvec$ in UCLS is the result of an empirical investigation into alternative linear system solvers.
We investigated using a Sherman-Morrison update, with exponential averaging (in Algorithm \ref{alg_LSTD_SM}) as well as improved incremental inverse updates, including one for pseudo-inverses \citep{meyer1973generalized}. This update has a confounding role for $\eta$, and for small $\eta$ we found it less stable than our proposed update.
We investigated iterative updates with a fixed stepsize, $\wvec_{t+1} = \wvec_{t} + \stepsize (\bvec_t - \Amat_t \wvec)$; the addition of the step-size, however, removes some of the parameter-free benefits of LSTD.
We investigated conjugate gradient updates, as in Algorithm \ref{alg_LSTD_CG}.
We finally derived the iterative update proposed, for $\Bmat \approx \Amat^\invt$, to obtain a preconditioner for the iterative update.

%UCLS is the result of experiments with alternative approaches for updating $\Amat$. While we utilize an update to UCLS that does not need the $\Ainv$ matrix, an alternative is to maintain the actual exponentially weighted $\Ainv$ matrix. Therefore, for completeness we also provide the necessary update in this case, and a version of LSTD we experimented with utilizing it. It must be noted that, utilizing this form of $\Ainv$ confounds the role of $\eta$ again, but in our limited experiments we have seen that a small value of $\eta$ (as necessary for playing the role of a regularizer despite being exponentially down-weighted) is sufficient.

%\subsection{Derivation of incremental Ainv}
For completeness, we include the derivation for the Sherman-Morrison update.
\iffalse
The derivation for $\Amat_{t+1}^{-1}$  using $\Amat_{t+1} = (1-\beta)\Amat_t + \beta u v ^T$ is as follows:
\begin{align*}
\Amat_t^{-1} \Amat_{t+1}
%&= (1-\beta)\Amat_t^{-1}\Amat_t + \beta \Amat_t^{-1}u v ^T\\
&= (1-\beta)\textbf{I} + \beta \Amat_t^{-1}u v ^T
\end{align*}
Converting it to terms of $\Amat_{t+1}^{-1}$:
\begin{align*}
\Amat_{t+1}^{-1} &= ((1-\beta)\textbf{I} + \beta \Amat_t^{-1}u v ^T)^{-1} \Amat_t^{-1}\\
=  \bigg(\frac{1}{(1-\beta)}\textbf{I} &- \frac{1}{1+\frac{\beta}{1-\beta}v^T \Amat_t^{-1} u}\frac{\beta}{(1-\beta)^2}\Amat_t^{-1}u v ^T)\bigg) \Amat_t^{-1} \\
%&= \frac{1}{(1-\beta)}\Amat_t^{-1} + \frac{\frac{\beta}{(1-\beta)^2}\Amat_t^{-1}uv^T\Amat_t^{-1}}{\frac{(1-\beta) + \beta v^T\Amat_t^{-1}u}{1-\beta}}\\
&= \frac{1}{(1-\beta)}\Amat_t^{-1} - \frac{\frac{\beta}{(1-\beta)}\Amat_t^{-1}uv^T\Amat_t^{-1}}{(1-\beta) + \beta v^T\Amat_t^{-1}u}
\end{align*}
The first step utilizes a Lemma in \citep{miller1981inverse}. % (Page 68 of the magazine).
\fi
The derivation for $\Amat_{t+1}^{-1}$  using $\Amat_{t+1} = (1-\beta)\Amat_t + \beta u v ^T$ is as follows:
\begin{align*}
\Amat_{t+1}^\inv &= ((1-\beta)\Amat_t + \beta u v ^T)^\inv \\
&\stackrel{(1)}{=} \frac{1}{1-\beta} \Amat_t^\inv - \frac{{\frac{\beta}{(1-\beta)^2} \Amat_t^{-1}u v^T \Amat_t^{-1}}}{1+tr\big(\frac{\beta}{1-\beta} u v^\top \Amat_t^\inv \big)}  \\
&\stackrel{(2)}{=} \frac{1}{1-\beta} \Amat_t^\inv - \frac{\frac{\beta}{1-\beta} \Amat_t^{-1}u v^T \Amat_t^{-1}}{1-\beta + \beta v^\top\Amat_t^\inv u} \\
&= \frac{1}{1-\beta} \bigg(\Amat_t^\inv - \frac{\beta \Amat_t^{-1}u v^T \Amat_t^{-1}}{1-\beta + \beta v^\top\Amat_t^\inv u} \bigg)
\end{align*}
where step (1) utilizes a Lemma in \citep{miller1981inverse}, and step (2) utilizes $tr(u v^\top \Amat_t^\inv) = v^\top \Amat_t^\inv u$. % (Page 68 of the magazine).

\begin{figure*}[!htb]
\includegraphics[width=\textwidth ,scale=0.25]{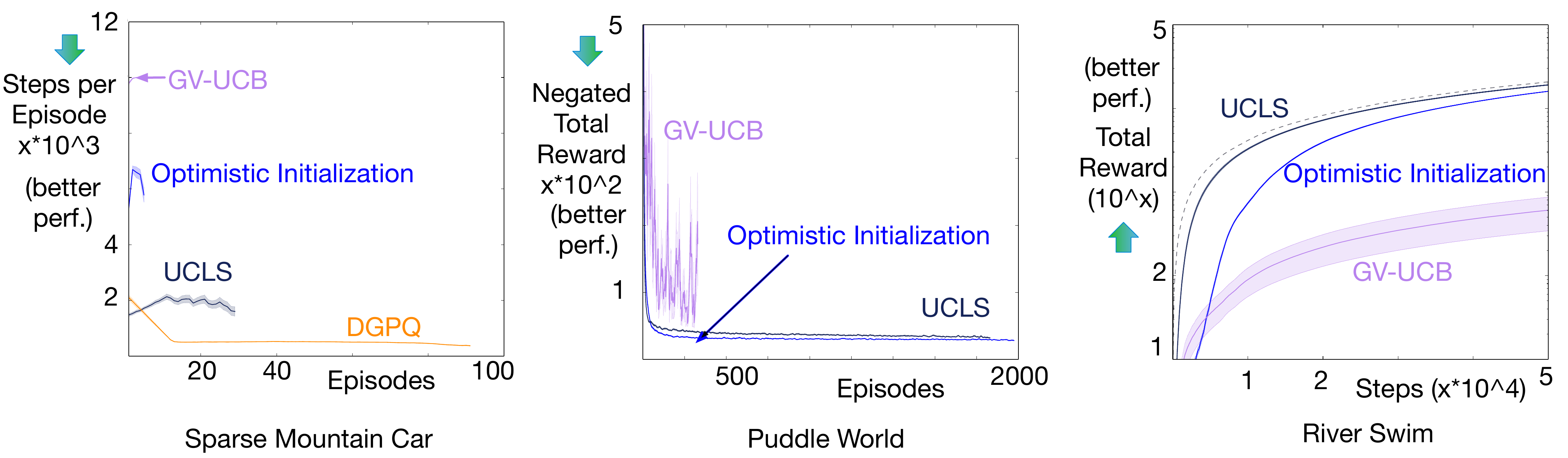}
\caption{Learning curves in the three domains comparing UCLS to additional methods. In the first two plots lower on y-axis indicates better performance, whereas in the right-most plot higher along y-axis is better.}
\label{fig:LC-APP}
\end{figure*}

\section{Extended results}\label{extResults}
%Original
% We show additional results here comparing UCLS to Sarsa with optimistic initialization and GV-UCB. We also show DGPQ for MCSparse. In the end, are plots showing best and worst runs for UCLS and DGPQ --- the two closest competitors --- to help better comprehend the variance of each algorithm.

%Edited
We show additional results here comparing UCLS to Sarsa with optimistic initialization and GV-UCB (p=0.5), Figure \ref{fig:LC-APP}; along with DGPQ in MCSparse. Also included are plots showing best and worst runs for UCLS and DGPQ --- the two closest competitors --- to show the variance of each algorithm Figure \ref{fig:RUNS}. Additionally, to empirically reinforce the utility of contextual confidence interval radius (CIR) over global CIR, we evaluate the policies obtained by UCLS and GV-UCB after 50,000 learning samples in River Swim and present the results in Figure \ref{fig:CIRevaluation}.

As mentioned in the main results, Sarsa with optimistic initialization performs remarkably well in these domains. In Sparse Mountain Car, as DGPQ converts the sparse reward dynamics to a dense one, it outperforms Sarsa with optimistic initialization as well.  In Puddle World, UCLS matches up to Sarsa's policy. In River Swim UCLS experiences minimal regret when compared to Sarsa's control policy. With the loss of contextual variance estimates GV-UCB explores the complete state space more thoroughly, and therefore performs poorly. The explicit upper confidence bound given by UCLS does not suffer from this, and sufficiently explores the domain to converge to an optimal policy without excessively exploring. For regions where there is low variance, the upper-confidence-bound converges more quickly to zero, whereas it remains higher in regions of uncertainty. Therefore, contextual variance estimates provide the flexibility of variable convergence based on the variance of the region, and global variance estimates decay too slowly. When the policies obtained by GV-UCB and UCLS are evaluated, it is clear that the policy obtained by UCLS is much closer to the optimal policy than the policy obtained by GV-UCB, showing that the exploration strategy used by UCLS is more \textit{data efficient}.

\includegraphics[width=\columnwidth]{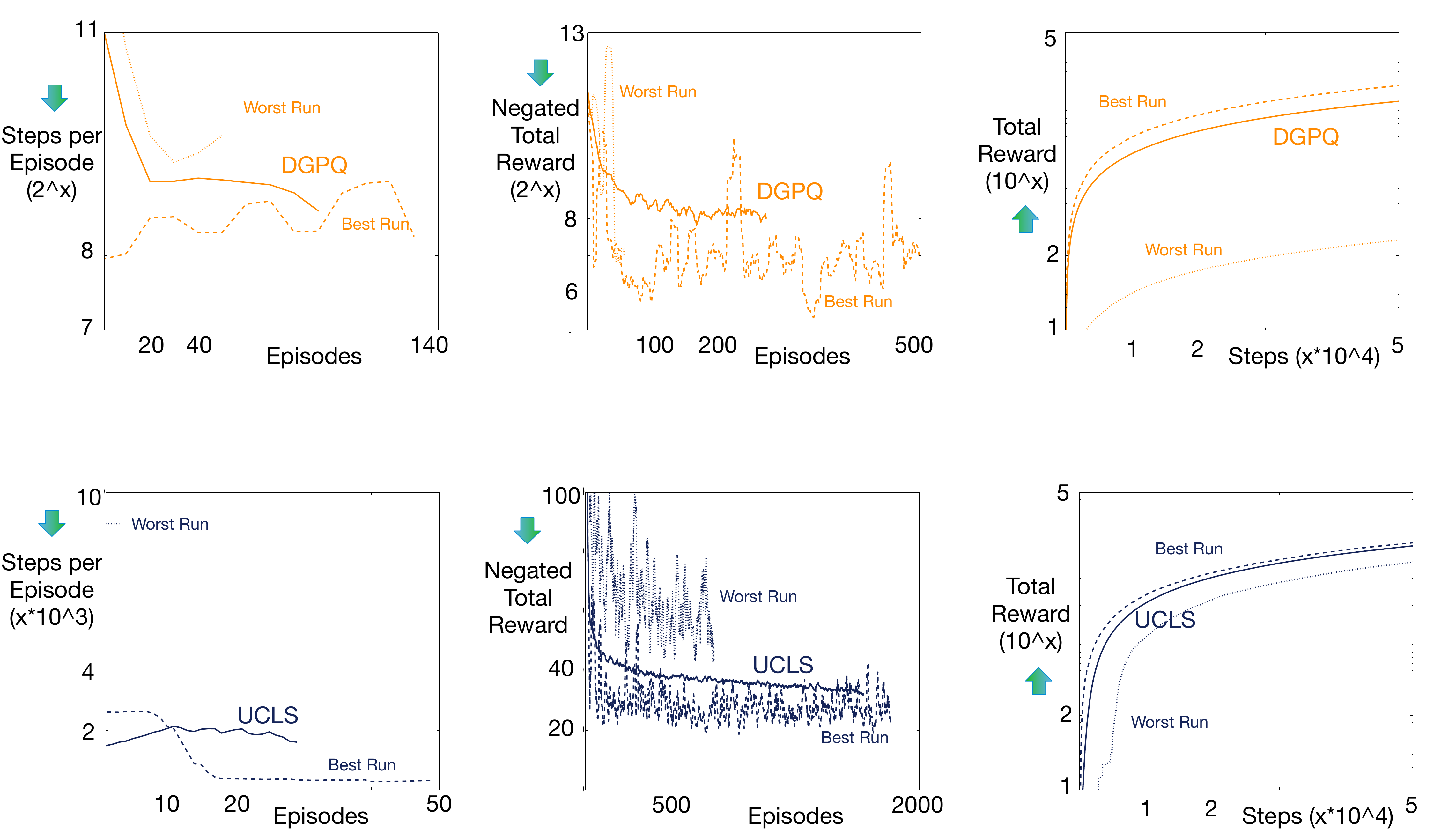}
\captionof{figure}{Best and worst run curves for DGPQ(top) and UCLS(bottom). From left to right: Sparse Mountain Car, Puddle World, River Swim.
\label{fig:RUNS}}
%\end{minipage}
\hfill
%\begin{minipage}{0.45\columnwidth}
\includegraphics[width=\columnwidth]{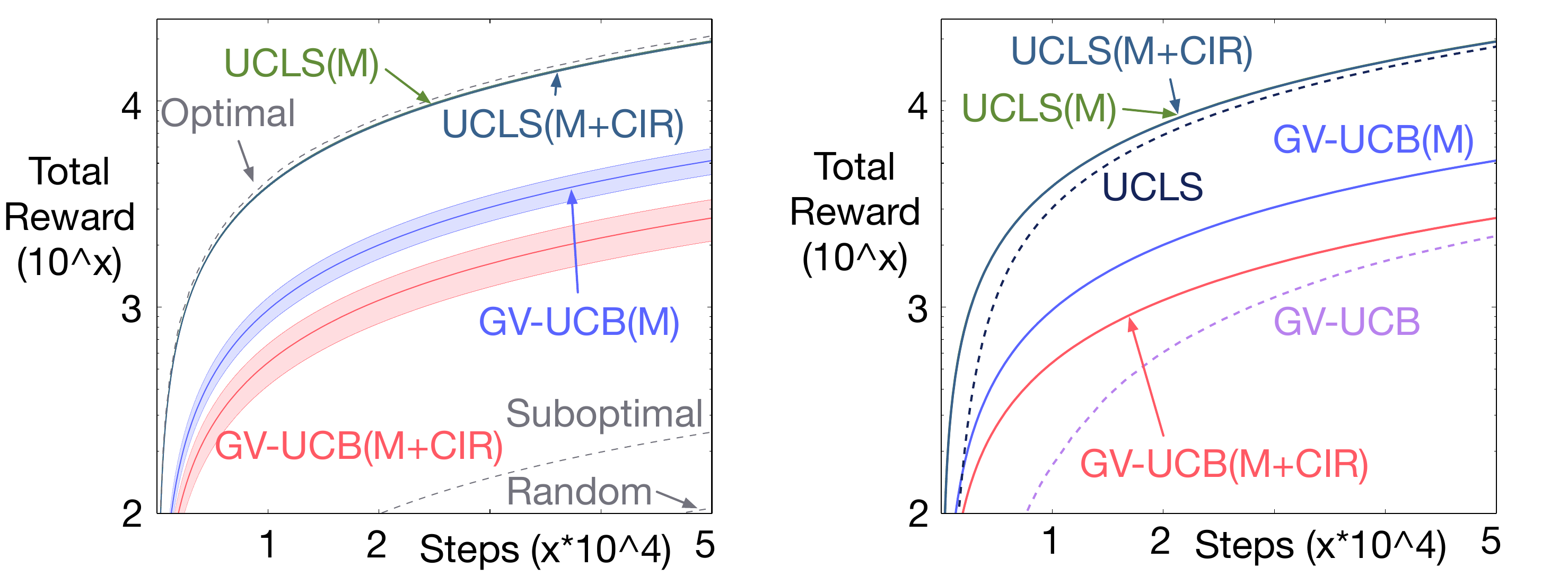}
\captionof{figure}{Policy evaluation plots comparing variations of final policy obtained by UCLS ($p=0.1$) and GV-UCB ($p=10e{-5}$) after 50,000 learning steps in River Swim. Policies with (M) indicate greedy policy w.r.t. mean estimates, whereas policies with (M+CIR) indicate greedy policies w.r.t. (mean + CIR) estimates. In the left plot it can be seen that UCLS(M) and UCLS(M+CIR) perform almost as well as the optimal policy, whereas both versions of GV-UCB are still sub-optimal in many parts of the state space. Additionally, the overlap of UCLS(M) and UCLS(M+CIR) indicates that contextual CIR fades faster than global CIR, and is a more data-efficient exploration strategy.
%Global CIR here leads to a worse performing policy as GV-UCB(M+CIR)'s performance is lesser than GV-UCB(M) - indicating promotion of undirected exploration by global CIR.
The right plot helps contrast the final policies obtained to the actual control policy used during learning (indicated by just UCLS and GV-UCB).
\label{fig:CIRevaluation}}

\end{document}